\documentclass[11pt]{article}

\usepackage[all=normal,floats, paragraphs]{savetrees}
\usepackage[square, comma, sort&compress]{natbib}

\usepackage{amsmath,amssymb,amsthm,amsfonts,latexsym,xspace,enumerate,graphicx,color,eucal,upgreek,mathtools}
\usepackage[backref,colorlinks,bookmarks=true]{hyperref}
\usepackage{float}

\usepackage[table]{xcolor}

\newtheorem{theorem}{Theorem}[section]
\newtheorem*{namedtheorem}{\theoremname}
\newcommand{\theoremname}{testing}

\newtheorem{lemma}[theorem]{Lemma}

\newtheorem{claim}[theorem]{Claim}

\newtheorem*{question*}{Question}

\theoremstyle{definition}
\newtheorem{definition}[theorem]{Definition}

\theoremstyle{plain}
\newtheorem{Alg}{Algorithm}

\newcommand{\dis}{\mbox{dis}}

\renewenvironment{proof}{\noindent{\textbf{Proof:}}} {$\blacksquare$\vskip \belowdisplayskip}

\newcommand{\ignore}[1]{}

\newcommand{\E}{\mathop{\bf E\/}}

\DeclareMathOperator{\spn}{span}

\newcommand{\R}{\mathbb R}

\newcommand{\eps}{\epsilon}

\newcommand{\norm}[1]{\left\lVert #1 \right\rVert}

\makeatletter
\floatstyle{ruled}
\newfloat{fragment}{H}{lop}
\floatname{fragment}{Algorithm}
\renewcommand{\floatc@ruled}[2]{\vspace{2pt}{\@fs@cfont #1.\:} #2 \par
 \vspace{1pt}}
\makeatother

\newcommand{\mypseudocodelabel}[1]{\hfil}

\usepackage{algpseudocode}
\usepackage{algorithm}
\newcommand{\prob}{p}

\newcommand{\argmin}{\operatornamewithlimits{argmin}}

\title{A Practical Algorithm for Topic Modeling with Provable Guarantees}

\begin{document} 

\title{A Practical Algorithm for Topic Modeling with Provable Guarantees}
\author{Sanjeev Arora \thanks{Princeton University, Computer Science Department and Center for Computational Intractability.
Email: {\tt arora@cs.princeton.edu}. This work is supported by the NSF grants CCF-0832797 and CCF-1117309.} \and
Rong Ge \thanks{Princeton University, Computer Science Department and Center for Computational Intractability.
Email: {\tt rongge@cs.princeton.edu}. This work is supported by the NSF grants CCF-0832797 and CCF-1117309.} \and
Yoni Halpern \thanks{New York University, Computer Science Department.
Email: {\tt halpern@cs.nyu.edu}. Research supported in part by an
NSERC Postgraduate Scholarship.} \and
David Mimno \thanks{Princeton University, Computer Science Department.
Email: {\tt mimno@cs.princeton.edu}. This work is supported by an NSF Computing and Innovation Fellowship.} \and
Ankur Moitra \thanks{Institute for Advanced Study, School of Mathematics.
Email: {\tt moitra@ias.edu}.
Research supported in part 
by NSF grant No.
DMS-0835373 and by an NSF Computing and Innovation Fellowship.} \and
David Sontag \thanks{New York University, Computer Science Department.
Email: {\tt dsontag@cs.nyu.edu}. Research supported by a Google Faculty Research Award, CIMIT award 12-1262, and grant UL1 TR000038 from NCATS, NIH.} \and
Yichen Wu \thanks{Princeton University, Computer Science Department.
Email: {\tt yichenwu@princeton.edu}.} \and
Michael Zhu \thanks{Princeton University, Computer Science Department.
Email: {\tt mhzhu@princeton.edu}.} 
}

\maketitle

\begin{abstract} 
Topic models provide a useful method for dimensionality reduction and exploratory data analysis in large text corpora. Most approaches to topic model inference have been based on a maximum likelihood objective. Efficient algorithms exist that approximate this objective, but they have no provable guarantees. Recently, algorithms have been introduced that provide provable bounds, but these algorithms are not practical because they are inefficient and not robust to violations of model assumptions. In this paper we present an algorithm for topic model inference that is both provable and practical. The algorithm produces results comparable to the best MCMC implementations while running orders of magnitude faster.
\end{abstract}

\section{Introduction}\label{sec:intro}

Topic modeling is a popular method that learns
thematic structure from large document collections without human supervision. 
The model is simple: documents are mixtures of topics, which are modeled as distributions over a vocabulary \citep{survey}.
Each word token is generated by selecting a topic from a document-specific distribution, and then selecting a specific word from that topic-specific distribution.
Posterior inference over document-topic and topic-word distributions is intractable --- in the worst case it is NP-hard even for just two topics~\citep{AGM}.
As a result, researchers have used approximate inference techniques such as 
singular value decomposition \citep{LSI}, variational inference \citep{LDA}, and MCMC \citep{griffithsFinding}.

Recent work in theoretical computer science focuses on designing {\em provably} efficient
algorithms for topic modeling. These treat the topic modeling problem
as one of {\em statistical recovery}: assuming the data was generated 
{\em perfectly} from the hypothesized model using an unknown set of
parameter values, the goal is to 
recover the model parameters in polynomial time given a reasonable number of samples. 

\citet{AGM} present an algorithm that provably
recovers the parameters of topic models 
provided that the topics meet a certain {\em separability} assumption \citep{DS}.
Separability requires that every topic contains at least one {\em anchor word} that has non-zero probability only in that topic. If a document contains this anchor word, then it is guaranteed that the corresponding topic is among the set of topics used to generate the document. 
The algorithm proceeds in two steps: first it selects anchor words for each topic; and second, in the recovery step, it reconstructs topic distributions given those anchor words.
The input for the algorithm is the second-order moment matrix of word-word co-occurrences.

\citet{AFHKL} present a provable algorithm based on third-order moments that does not require separability, but, unlike the algorithm of Arora et al., assumes that topics are not correlated.
Although standard topic models like LDA \citep{LDA} assume that the choice of topics used to generate the document are uncorrelated, there is strong evidence that topics are dependent~\citep{BL1,LM}: economics and politics are more likely to co-occur than economics and cooking.

Both algorithms run in polynomial time, but the bounds that have been proven on their sample complexity are weak and their empirical runtime performance is slow.
The algorithm presented by \citet{AGM} solves numerous linear programs to find anchor words.
\citet{BRRT} and \citet{Gillis} reduce the number of linear programs needed. 
All of these algorithms infer topics given anchor words using matrix inversion, which is notoriously unstable and noisy: 
matrix inversion frequently generates negative values for topic-word probabilities.

In this paper we present three contributions.
First, we replace linear programming with a combinatorial anchor selection algorithm.
So long as the separability assumption holds, we prove that this algorithm is stable in the presence of noise and thus has polynomial sample complexity for learning topic models. 
Second, we present a simple probabilistic interpretation
of topic recovery given anchor words that replaces matrix inversion with a new gradient-based inference method.
Third, we present an empirical comparison between recovery-based algorithms and existing likelihood-based topic inference. We study both the empirical sample complexity of the algorithms on synthetic distributions and the performance of the algorithms on real-world document corpora. We find that our algorithm performs as well as collapsed Gibbs sampling on a variety of metrics, and runs at least an order of magnitude faster.

Our algorithm both inherits the provable guarantees from \citet{AGKM, AGM} and also results in simple, practical implementations.
We view our work as a step toward bridging the gap between statistical recovery approaches to machine learning and maximum likelihood estimation, allowing us to circumvent the computational intractability of maximum likelihood estimation yet still be robust to model error.

\section{Background}
\label{sec:background}



We consider the learning problem for a class of admixture distributions that are frequently used for probabilistic topic models. Examples of such distributions include latent Dirichlet allocation \citep{LDA}, correlated topic models \citep{BL1}, and Pachinko allocation \citep{LM}. We denote the number of words in the vocabulary by $V$ and the number of topics by $K$. Associated with each topic $k$ is a multinomial distribution over the words in the vocabulary, which we will denote as the column vector $A_k$ of length $V$. Each of these topic models postulates a particular prior distribution $\tau$ over the topic distribution of a document. For example, in latent Dirichlet allocation (LDA) $\tau$ is a Dirichlet distribution, and for the correlated topic model $\tau$ is a logistic Normal distribution.
The generative process for a document $d$ begins by drawing the document's topic distribution $W_d \sim  \tau$. Then, for each position $i$ we sample a topic assignment $z_i \sim W_d$, and finally a word $w_i \sim A_{z_i}$.


We can combine the column vectors $A_k$ for each of the $K$ topics to obtain the word-topic matrix $A$ of dimension $V \times K$.
We can similarly combine the column vectors $W_d$ for $M$ documents to obtain the topic-document matrix $W$ of dimension $K \times M$. We emphasize that $W$ is unknown and stochastically generated: we can never expect to be able to recover it. The learning task that we consider is to find the word-topic matrix $A$. For the case when $\tau$ is Dirichlet (LDA), we also show how to learn hyperparameters of $\tau$.


Maximum likelihood estimation of the word-topic distributions is NP-hard even for two topics \citep{AGM}, and as a result researchers typically use approximate inference. The most popular approaches are variational inference \citep{LDA}, which optimizes an approximate objective, and Markov chain Monte Carlo \citep{MALLET}, which asymptotically samples from the posterior distribution but has no guarantees of convergence.

\citet{AGM} present an algorithm that provably learns the parameters of a topic model given samples from the model, provided that the word-topic distributions satisfy a condition called {\em separability}:
\begin{definition}
The word-topic matrix $A$ is $p$-separable for $p > 0$ if for each topic $k$, there is some word $i$ such that $A_{i,k}\geq p$ and $A_{i,k'}=0$ for $k'\neq k$. 
\end{definition}
Such a word is called an {\em anchor word} because when it occurs in a document, it is a perfect indicator that the document is at least partially about the corresponding topic, since there is no other topic that could have generated the word. 
Suppose that each document is of length $D \geq 2$, and let $R = \mathbb{E}_\tau[W W^T]$ be the $K\times K$ topic-topic covariance matrix.
Let $\alpha_k$ be the expected proportion of topic $k$ in a document generated according to $\tau$. 
The main result of \citet{AGM} is:
\begin{theorem}
There is a polynomial time algorithm that learns the parameters of a topic model
if the number of documents is at least 
$$M = \max\left\{O\left(\frac{\log V \cdot a^4K^6}{\epsilon^2p^6\gamma^2 D}\right), O\left(\frac{\log K \cdot a^2 K^4}{\gamma^2}\right)\right\},$$
where $p$ is defined above, $\gamma$ is the condition number of $R$, and $a=\max_{k, k'} \alpha_k/\alpha_{k'}$.
The algorithm learns the word-topic matrix $A$ and the topic-topic covariance matrix $R$ up to additive error $\epsilon$.
\end{theorem}

Unfortunately, this algorithm is not practical. Its running time is prohibitively large because it solves $V$ linear programs, and its use of matrix inversion makes it unstable and sensitive to noise. In this paper, we will give various reformulations and modifications of this algorithm that alleviate these problems altogether.

\section{A Probabilistic Approach to Exploiting Separability}
\label{sect:probabilistic}

\begin{algorithm}[t]
\caption{High Level Algorithm}
\label{alg:high_level}    
\begin{algorithmic}
\Require Textual corpus $\mathcal{D}$, Number of anchors $K$, Tolerance parameters $\epsilon_a,  \epsilon_b > 0$. 
\Ensure Word-topic matrix $A$, topic-topic matrix $R$
\State $Q \gets \text{Word Co-occurences}(\mathcal{D})$
\State Form $\{\bar{Q}_1, \bar{Q}_2, ... \bar{Q}_V\}$, the normalized rows of $Q$. 
\State $\mathbf{S}$ $\gets$ FastAnchorWords($\{\bar{Q}_1, \bar{Q}_2, ... \bar{Q}_V\}$, $K$, $\epsilon_a$) (Algorithm~\ref{alg:anchorword})
\State $A, R \gets$ RecoverKL($Q, \mathbf{S}, \epsilon_b$) (Algorithm~\ref{alg:NNR})
\State \Return $A, R$
\end{algorithmic}
\end{algorithm}

The \citet{AGM} algorithm has two steps: {\em anchor selection}, which identifies anchor words, and {\em recovery}, which recovers the parameters of $A$ and of $\tau$. Both anchor selection and recovery take as input the matrix $Q$ (of size $V \times V$) of word-word co-occurrence counts, whose construction is described in the supplementary material. $Q$ is normalized so that the sum of all entries is $1$.
The high-level flow of our complete learning algorithm is described in Algorithm~\ref{alg:high_level}, and follows the same two steps. In this section we will introduce a new recovery method based on a probabilistic framework.
We defer the discussion of anchor selection to the next section, where we provide a purely combinatorial algorithm for finding the anchor words. 

The original recovery procedure (which we call ``Recover'') from \citet{AGM} is as follows.
First, it permutes the $Q$ matrix so that the first $K$ rows and columns correspond to the anchor words. We will use the notation $Q_{\mathbf{S}}$ to refer to the first $K$ rows, and $Q_{\mathbf{S}, \mathbf{S}}$ for the first $K$ rows and just the first $K$ columns. If constructed from infinitely many documents, $Q$ would be the second-order moment matrix $Q = \mathbb{E}[AWW^TA^T] = A\mathbb{E}[WW^T]A^T = ARA^T$, with the following block structure:
\begin{equation}
Q = ARA^T = \begin{pmatrix} D\\U \end{pmatrix} R \begin{pmatrix} D & U^T \end{pmatrix} = 
\begin{pmatrix} 
DRD & DRU^T\\
URD & URU^T \end{pmatrix} \nonumber
\end{equation}
where $D$ is a diagonal matrix of size $K \times K$. Next, it solves for $A$ and $R$ using the algebraic manipulations outlined in Algorithm~\ref{alg:original_recovery}. 

\begin{algorithm}[t]
\caption{Original Recover \citep{AGM}}
\label{alg:original_recovery}
\begin{algorithmic}
\Require Matrix $Q$, Set of anchor words $\mathbf{S}$ 
\Ensure Matrices $A$,$R$
\State Permute rows and columns of $Q$
\State Compute $\vec{p}_{\mathbf{S}} = Q_{\mathbf{S}} \vec{1}$ \hfill (equals $DR\vec{1}$)
\State Solve for $\vec{z}$: $Q_{\mathbf{S},\mathbf{S}}\vec{z} = \vec{p}_{\mathbf{S}}$ \hfill ($\text{Diag}(\vec{z})$ equals $D^{-1}$)
\State Solve for $A^T$ = $(Q_{\mathbf{S},\mathbf{S}}\text{Diag}(\vec{z}))^{-1}Q_{\mathbf{S}}^T$
\State Solve for $R = \text{Diag}(\vec{z})Q_{\mathbf{S},\mathbf{S}}\text{Diag}(\vec{z})$ 
\State \Return $A, R$
\end{algorithmic}
\end{algorithm}


The use of matrix inversion in Algorithm~\ref{alg:original_recovery} results in substantial imprecision in the estimates when we have small sample sizes. The returned $A$ and $R$ matrices can even contain small negative values, requiring a subsequent projection onto the simplex. As we will show in Section~\ref{sec:results}, the original recovery method performs poorly relative to a likelihood-based algorithm. Part of the problem is that the original recover algorithm uses only $K$ rows of the matrix $Q$ (the rows for the anchor words), whereas $Q$ is of dimension $V \times V$. Besides ignoring most of the data, this has the additional complication that it relies completely on co-occurrences between a word and the anchors, and this estimate may be inaccurate if both words occur infrequently.


Here we adopt a new {\em probabilistic} approach, which we describe below after introducing some notation. 
Consider any two words in a document and call them $w_1$ and $w_2$, and let $z_1$ and $z_2$ refer to their topic assignments.
We will use $A_{i,k}$ to index the matrix of word-topic distributions, i.e. $A_{i,k} = \prob(w_1 = i | z_1 = k) = \prob(w_2 = i | z_2 = k)$. Given infinite data, the elements of the $Q$ matrix can be interpreted as $Q_{i,j} = \prob(w_1 = i, w_2 = j)$. The row-normalized $Q$ matrix, denoted $\bar{Q}$, which plays a role in both finding the anchor words and the recovery step, can be interpreted as a conditional probability $\bar{Q}_{i,j} = \prob(w_2 = j | w_1 = i)$.
\begin{algorithm}[t]
\caption{RecoverKL}
\label{alg:NNR}    
\begin{algorithmic}
\Require Matrix $Q$, Set of anchor words $\mathbf{S}$, tolerance parameter $\epsilon$.  
\Ensure Matrices $A$,$R$
\State Normalize the rows of $Q$ to form $\bar{Q}$.
\State Store the normalization constants $\vec{p}_w = Q\vec{1}$.
\State $\bar{Q}_{s_k}$ is the row of $\bar{Q}$ for the $k^{th}$ anchor word. 
\For{$i = 1, ..., V$}
\State Solve $C_{i \cdot} = \argmin_{\vec{C}_i} D_{KL}(\bar{Q}_{i} || \sum_{k\in \mathbf{S}} C_{i,k} \bar{Q}_{s_k})$
\State Subject to: $\sum_k C_{i,k} = 1$ and $C_{i,k} \geq 0$
\State With tolerance: $\epsilon$
\EndFor
\State $A^\prime = \text{diag}(\vec{p}_w)C$
\State Normalize the columns of $A^\prime$ to form $A$. 
\State $R = A^\dagger Q {A^\dagger}^T$
\State \Return $A, R$
\end{algorithmic}
\end{algorithm}

Denoting the indices of the anchor words as $\mathbf{S} = \{s_1, s_2, ..., s_K\}$, the rows indexed by elements of $\mathbf{S}$ are special in that every other row of $\bar{Q}$ lies in the convex hull of the rows indexed by the anchor words. To see this, first note that for an anchor word $s_k$,
\begin{eqnarray}
\hspace{-4mm}\bar{Q}_{s_k, j}\hspace{-2mm}&=&\hspace{-3mm} \sum_{k^\prime} \prob(z_1 = k^\prime | w_1 = s_k) \prob(w_2 = j | z_1=k^\prime )\label{eq:anchor_step1}\\
\hspace{-6mm}&=&\hspace{-3mm} \prob(w_2 = j | z_1=k),\label{eq:anchor_step2}
\end{eqnarray}
where \eqref{eq:anchor_step1} uses the fact that in an admixture model $w_2 \bot w_1 \mid z_1$, and \eqref{eq:anchor_step2} is because $\prob(z_1 = k | w_1 = s_k) = 1$. For any other word $i$, we have
\begin{eqnarray*}
\bar{Q}_{i, j} &=& \sum_{k} \prob(z_1 = k| w_1=i)\prob(w_2 = j | z_1=k).
\end{eqnarray*}
Denoting the probability $\prob(z_1 = k | w_1 = i)$ as $C_{i,k}$, we have $\bar{Q}_{i, j} = \sum_k C_{i,k} \bar{Q}_{s_k, j}$. Since $C$ is non-negative and $\sum_k C_{i,k} = 1$, we have that any row of $\bar{Q}$ lies in the convex hull of the rows corresponding to the anchor words. The mixing weights give us $\prob(z_1 | w_1 = i)$! Using this together with $\prob(w_1 = i)$, we can recover the $A$ matrix simply by using Bayes' rule:
$$
\prob(w_1 = i | z_1 = k) = \frac{\prob(z_1 = k | w_1 = i)\prob(w_1 = i)}{\sum_{i^\prime} \prob(z_1 = k | w_1 = i^\prime)p(w_1 = i^\prime) }.
$$
Finally, we observe that $\prob(w_1 = i)$ is easy to solve for since $\sum_j Q_{i,j} = \sum_j p(w_1 = i, w_2 = j) = \prob(w_1 = i)$. 

Our new algorithm finds, for each row of the empirical row normalized co-occurrence matrix, $\hat{Q}_i$, the coefficients $\prob(z_1 | w_1 = i)$ that best reconstruct it as a convex combination of the rows that correspond to anchor words. This step can be solved quickly and in parallel (independently) for each word using the exponentiated gradient algorithm. Once we have $\prob(z_1 | w_1)$, we recover the $A$ matrix using Bayes' rule. The full algorithm using KL divergence as an objective is found in Algorithm~\ref{alg:NNR}. Further details of the exponentiated gradient algorithm are given in the supplementary material.


One reason to use KL divergence as the measure of reconstruction error is that the recovery procedure can then be understood as maximum likelihood estimation.
In particular, we seek the parameters $\prob(w_1)$, $\prob(z_1 | w_1)$, $\prob(w_2 | z_1)$ that maximize the likelihood of observing the {\em word co-occurence counts}, $\hat{Q}$.
However, the optimization problem does not explicitly constrain the parameters to correspond an admixture model.

We can also define a similar algorithm using quadratic loss, which we call RecoverL2.
This formulation has the extremely useful property that both the objective and gradient can be kernelized so that the optimization problem 
is independent of the vocabulary size. To see this, notice that the objective can be re-written as
$$||\overline{Q}_i - C_i^T \overline{Q}_{\mathbf{S}}||^2 =  ||\overline{Q}_i||^2  - 2C_i(\overline{Q}_{\mathbf{S}} \overline{Q}_i^T)  + C_i^T(\overline{Q}_{\mathbf{S}}\overline{Q}^T_{\mathbf{S}})C_i,$$
where $\overline{Q}_{\mathbf{S}}\overline{Q}^T_{\mathbf{S}}$ is $K \times K$ and can be computed once and used for all words,  and $\overline{Q}_{\mathbf{S}} \overline{Q}_i^T$ is $K \times 1$ and can be computed once prior to running the exponentiated gradient algorithm for word $i$.

To recover the $R$ matrix for an admixture model, 
recall that $Q = ARA^T$.
This may be an over-constrained system of equations with no solution for $R$, but we can find a least-squares approximation to $R$ by pre- and post-multiplying $Q$ by the pseudo-inverse $A^\dagger$. 
For the special case of LDA we can learn the Dirichlet hyperparameters. Recall that in applying Bayes' rule we calculated $\prob(z_1) = \sum_{i^\prime} \prob(z_1 | w_1 = i^\prime)\prob(w_1 = i^\prime)$.  These values for $\prob(z)$ specify the Dirichlet hyperparameters up to a constant scaling. This constant could be recovered from the $R$ matrix \citep{AGM}, but in practice we find it is better to choose it using a grid search to maximize the likelihood of the training data.

We will see in Section~\ref{sec:results} that our nonnegative recovery algorithm performs much better on a wide range of performance metrics than the recovery algorithm in \citet{AGM}. In the supplementary material we show that it also inherits the theoretical guarantees of \citet{AGM}: given polynomially many documents, our algorithm returns an estimate $\hat{A}$ at most $\epsilon$ from the true word-topic matrix $A$.


\section{A Combinatorial Algorithm for Finding Anchor Words}

Here we consider the {\em anchor selection} step of the algorithm where our goal is to find the anchor words. In the {\em infinite data} case where we have infinitely many documents, the convex hull of the rows in $\overline{Q}$ will be a simplex where the vertices of this simplex correspond to the anchor words. 
Since we only have a finite number of documents, the rows of $\overline{Q}$ are only an approximation to their expectation. We are therefore given a set of $V$ points $d_1, d_2, ... d_V$ that are each a perturbation of $a_1, a_2, ... a_V$ whose convex hull $P$ defines a simplex. We would like to find an approximation to the vertices of $P$. See \citet{AGKM} and \citet{AGM} for more details about this problem. 

\begin{algorithm}[t]
\caption{FastAnchorWords\label{algorithm:fasteranchors}}
\label{alg:anchorword}
{\bf Input:} $V$ points $\{d_1, d_2, ... d_V\}$  in $V$ dimensions, almost in a simplex with $K$ vertices and $\epsilon > 0$

{\bf Output:} $K$ points that are close to the vertices of the simplex.

\begin{algorithmic}
\State Project the points $d_i$ to a randomly chosen $4 \log V/\epsilon^2$ dimensional subspace
\State $S \leftarrow \{d_i\}$ s.t. $d_i$ is the farthest point from the origin.
\For{$i = 1$ TO $K-1$}
\State Let $d_j$ be the point in $\{d_1, \ldots, d_V\}$ that has the largest distance to $\spn(S)$.
\State $S \leftarrow S\cup \{d_j\}$.
\EndFor
\State $S = \{v'_1, v'_2, ... v'_K\}$.
\For {$i = 1$ TO $K$}
\State Let $d_j$ be the point that has the largest distance to $\spn(\{v_1',v_2',...,v_K'\}\backslash \{v_i'\})$
\State Update $v'_i$ to $d_j$
\EndFor
\State Return $\{v_1', v_2', ..., v_K'\}$.

\end{algorithmic}
\hrulefill

{\small Notation: $\spn(S)$ denotes the subspace spanned by the points in the set $S$. We compute the distance from a point $x$ to the subspace $\spn(S)$ by computing the norm of the projection of $x$ onto the orthogonal complement of $\spn(S)$. }
\end{algorithm}

\citet{AGKM} give a polynomial time algorithm that finds the anchor words. However, their algorithm is based on solving $V$ linear programs, one for each word, to test whether or not a point is a vertex of the convex hull. In this section we describe a purely combinatorial algorithm for this task that avoids linear programming altogether. The new ``FastAnchorWords'' algorithm is given in Algorithm~\ref{algorithm:fasteranchors}. 
To find all of the anchor words, our algorithm iteratively finds the furthest point from the subspace spanned by the anchor words found so far.

Since the points we are given are perturbations of the true points, we cannot hope to find the anchor words exactly. 
Nevertheless, the intuition is that even if one has only found $r$ points $S$ that are close to $r$ (distinct) anchor words, the point that is furthest from $\spn(S)$ will itself be close to a (new) anchor word. The additional advantage of this procedure is that when faced with many choices for a next anchor word to find, our algorithm tends to find the one that is most different than the ones we have found so far. 



The main contribution of this section is a proof that the FastAnchorWords algorithm succeeds in finding $K$ points that are close to anchor words. To precisely state the guarantees, we recall the following definition from \citep{AGKM}:

\begin{definition}
A simplex $P$ is $\gamma$-robust if for every vertex $v$ of $P$, the $\ell_2$ distance between $v$ and the convex hull of the rest of the vertices is at least $\gamma$.
\end{definition}

\noindent In most reasonable settings the parameters of the topic model define lower bounds on the robustness of the polytope $P$. For example, in LDA, this lower bound is based on the largest ratio of any pair of hyper-parameters in the model \citep{AGM}. Our goal is to find a set of points that are close to the vertices of the simplex, and to make this precise we introduce the following definition:

\begin{definition}
Let $a_1, a_2, ... a_V$ be a set of points whose convex hull $P$ is a simplex with vertices $v_1, v_2, ... v_K$. Then we say $a_i$ $\epsilon$-covers $v_j$ if when $a_j$ is written as a convex combination of the vertices as $a_i = \sum_j c_j v_j$, then $c_j \geq 1-\epsilon$. Furthermore we will say that a set of $K$ points $\epsilon$-covers the vertices if each vertex is $\epsilon$ covered by some point in the set. 
\end{definition}

We will prove the following theorem: suppose there is a set of points $\mathcal{A} = a_1, a_2, ... a_V$ whose convex hull $P$ is $\gamma$-robust and has vertices $v_1, v_2, ... v_K$ (which appear in $\mathcal{A}$) and that we are given a perturbation $d_1, d_2, ... d_V$ of the points so that for each $i$, $\|a_i - d_i\| \leq \epsilon$, then:

\begin{theorem}
There is a combinatorial algorithm that runs in time $\tilde{O}(V^2+VK/\epsilon^2)$
\footnote{In practice we find setting dimension to 1000 works well. The running time is then $O(V^2+1000VK)$.} and outputs a subset of $\{d_1,\ldots, d_V\}$ of size $K$ that $O(\epsilon/\gamma)$-covers the vertices provided that  $20 K \epsilon/\gamma^2 < \gamma$.
\end{theorem}

This new algorithm not only helps us avoid linear programming altogether in inferring the parameters of a topic model, but also can be used to solve the nonnegative matrix factorization problem under the separability assumption, again without resorting to linear programming. Our analysis rests on the following lemmas, whose proof we defer to the supplementary material. Suppose the algorithm has found a set $S$ of $k$ points that are each $\delta$-close to distinct vertices in $\{v_1, v_2, ..., v_K\}$ and that $\delta < \gamma/20K$. 

\begin{lemma}
\label{lem:perturbation}
There is a vertex $v_i$ whose distance from $\spn(S)$ is at least $\gamma/2$.
\end{lemma}

The proof of this lemma is based on a volume argument, and the connection between the volume of a simplex and the determinant of the matrix of distances between its vertices. 

\begin{lemma}
\label{lem:inductionforfastalg}
The point $d_j$ found by the algorithm must be $\delta = O(\epsilon/\gamma^2)$ close to some vertex $v_i$.
\end{lemma}

This lemma is used to show that the error does not accumulate too badly in our algorithm, since $\delta$ only depends on $\epsilon$, $\gamma$ (not on the $\delta$ used in the previous step of the algorithm). This prevents the error from accumulating exponentially in the dimension of the problem, which would be catastrophic for our proof. 

After running the first phase of our algorithm, we run a cleanup phase (the second loop in Alg. \ref{alg:anchorword}) that can reduce the error in our algorithm. When we have $K-1$ points close to $K-1$ vertices, only one of the vertices can be far from their span. The farthest point must be close to this missing vertex. The following lemma shows that this cleanup phase can improve the guarantees of  Lemma~\ref{lem:inductionforfastalg}:

\begin{lemma}\label{lem:cleanup}
Suppose $|S| = K-1$ and each point in $S$ is $\delta = O(\epsilon/\gamma^2) < \gamma/20K$ close to some vertex $v_i$, then the farthest point $v_j'$ found by the algorithm is $1 - O(\epsilon/\gamma)$ close to the remaining vertex.
\end{lemma}

This algorithm is a greedy approach to maximizing the volume of the simplex. The larger the volume is, the more words per document the resulting model can explain. Better anchor word selection is an open question for future work. We have experimented with a  variety of other heuristics for maximizing  simplex volume, with varying degrees of success.

{\bf Related work.}
The separability assumption has also been studied under the name ``pure pixel assumption" in the context of hyperspectral unmixing. A number of algorithms have been proposed that overlap with ours -- such as the VCA \citep{VCA} algorithm (which differs in that there is no clean-up phase) and the N-FINDR \citep{NFINDR} algorithm which attempts to greedily maximize the volume of a simplex whose vertices are data points. However these algorithms have only been proven to work in the infinite data case, and for our algorithm we are able to give provable guarantees even when the data points are perturbed (e.g., as the result of sampling noise). Recent work of \citet{thurau10cikm} and \citet{Kumar12} follow the same pattern as our paper, but use 
non-negative matrix factorization under the separability assumption. While 
both give applications to topic modeling, in realistic applications the term-by-document matrix is too sparse to be considered a good approximation to its expectation (because documents are short). In contrast, our algorithm 
works with the Gram matrix $Q$ so that we can give provable guarantees even when each document is short. 

\section{Experimental Results}
\label{sec:results}

We compare three parameter recovery methods, Recover \citep{AGM}, RecoverKL and RecoverL2 to a fast implementation of Gibbs sampling \citep{MALLET}.\footnote{We were not able to obtain \citet{AFHKL}'s implementation of their algorithm, and our own implementation is too slow to be practical.}
Linear programming-based anchor word finding is too slow to be comparable, so we use FastAnchorWords for all three recovery algorithms.
Using Gibbs sampling we obtain the word-topic distributions by averaging over 10 saved states, each separated by 100 iterations, after 1000 burn-in iterations.

\subsection{Methodology}

We train models on two synthetic data sets to evaluate performance when model assumptions are correct, and real documents to evaluate real-world performance.
To ensure that synthetic documents resemble the dimensionality and sparsity characteristics of real data, we generate {\em semi-synthetic} corpora.
For each real corpus, we train a model using MCMC and then generate new documents using the parameters of that model (these parameters are {\em not} guaranteed to be separable). 

We use two real-world data sets, a large corpus of {\bf New York Times} articles (295k documents, vocabulary size 15k, mean document length 298) and a small corpus of {\bf NIPS} abstracts (1100 documents, vocabulary size 2500, mean length 68).
Vocabularies were pruned with document frequency cutoffs.
We generate semi-synthetic corpora of various sizes from models trained with $K=100$ from NY Times and NIPS, with document lengths set to 300 and 70, respectively, and with document-topic distributions drawn from a Dirichlet with symmetric hyperparameters $0.03$.


We use a variety of metrics to evaluate models:
For the semi-synthetic corpora, we can compute {\bf reconstruction error} between the true word-topic matrix $A$ and learned topic distributions.
Given a learned matrix $\hat{A}$ and the true matrix $A$, we use an LP to find the best matching between topics.
Once topics are aligned, we evaluate $\ell_1$ distance between each pair of topics.
When true parameters are not available, a standard evaluation for topic models is to compute {\bf held-out probability}, the probability of previously unseen documents under the learned model.
This computation is intractable but there are reliable approximation methods \citep{wallach2009evaluation,buntine2009estimating}.
Topic models are useful because they provide interpretable latent dimensions.
We can evaluate the {\bf semantic quality} of individual topics using a metric called {\em Coherence}.
Coherence is based on two functions, $D(w)$ and $D(w_1, w_2)$, which are number of  documents with at least one instance of $w$, and of $w_1$ and $w_2$, respectively  \citep{mimno2011optimizing}.
Given a set of words $\mathcal{W}$, coherence is
\begin{align}
\label{eqn:pmi}
Coherence(\mathcal{W}) & = \sum_{w_1, w_2 \in \mathcal{W}} \log \frac{D(w_1, w_2) + \epsilon}{D(w_2)}.
\end{align}
The parameter $\epsilon=0.01$ is used to avoid taking the $\log$ of zero for words that never co-occur \citep{stevens2012exploring}.
This metric has been shown to correlate well with human judgments of topic quality.
If we perfectly reconstruct topics, all the high-probability words in a topic should co-occur frequently, otherwise, the model may be mixing unrelated concepts.
Coherence measures the quality of individual topics, but does not measure redundancy, so we measure {\bf inter-topic similarity}.
For each topic, we gather the set of the $N$ most probable words.
We then count how many of those words do not appear in any other topic's set of $N$ most probable words.
Some overlap is expected due to semantic ambiguity, but lower numbers of unique words indicate less useful models.

\subsection{Efficiency}

The Recover algorithms, in Python, are faster than a heavily optimized Java Gibbs sampling implementation \citep{yao2009efficient}.
\begin{figure}[tbp]
\begin{center}
\includegraphics[scale=0.35]{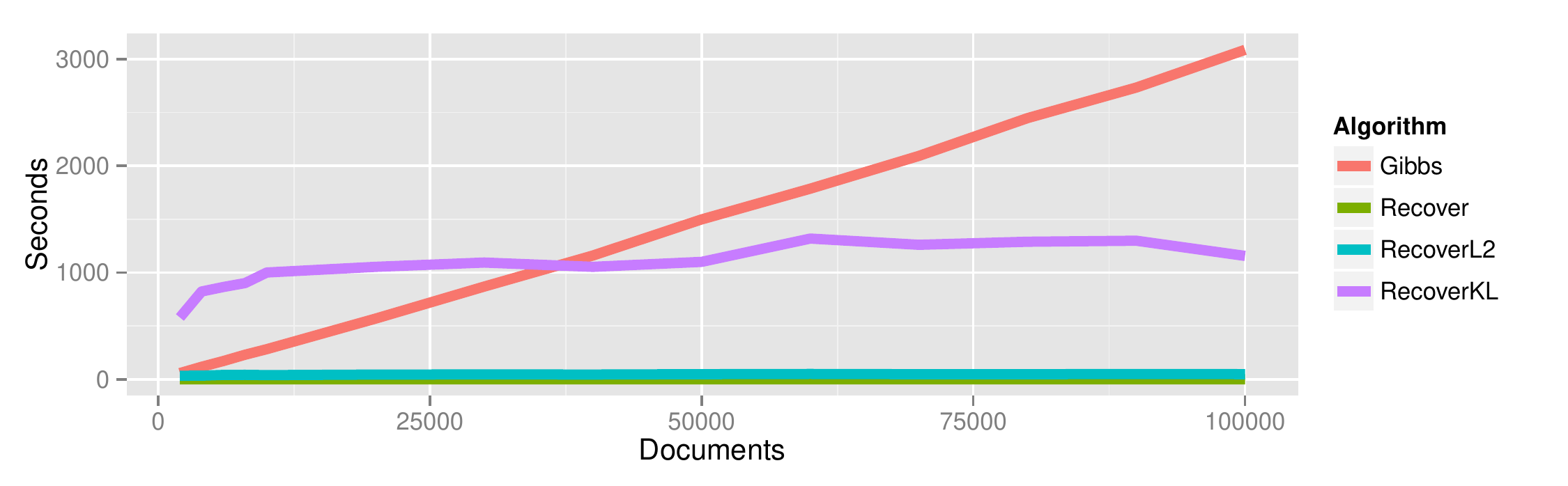}
\caption{Training time on synthetic NIPS documents.}
\label{fig:timing}
\end{center}
\end{figure}
Fig. \ref{fig:timing} shows the time to train models on synthetic corpora on a single machine.
Gibbs sampling is linear in the corpus size. 
RecoverL2 is also linear ($\rho=0.79$), but only varies from 33 to 50 seconds.
Estimating $Q$ is linear, but takes only 7 seconds for the largest corpus.
FastAnchorWords takes less than 6 seconds for all corpora.

\subsection{Semi-synthetic documents}

The new algorithms have good $\ell_1$ reconstruction error on semi-synthetic documents, especially for larger corpora.
Results for semi-synthetic corpora drawn from topics trained on NY Times articles are shown in Fig. \ref{fig:nyt-ss-l1} for corpus sizes ranging from 50k to 2M synthetic documents.
In addition, we report results for the three Recover algorithms on ``infinite data,'' that is, the true $Q$ matrix from the model used to generate the documents.
Error bars show variation between topics.
Recover performs poorly in all but the noiseless, infinite data setting.
Gibbs sampling has lower $\ell_1$ with smaller corpora, while the new algorithms get better recovery and lower variance with more data (although more sampling might reduce MCMC error further).
\begin{figure}[tbp]
\begin{center}
\includegraphics[scale=0.35]{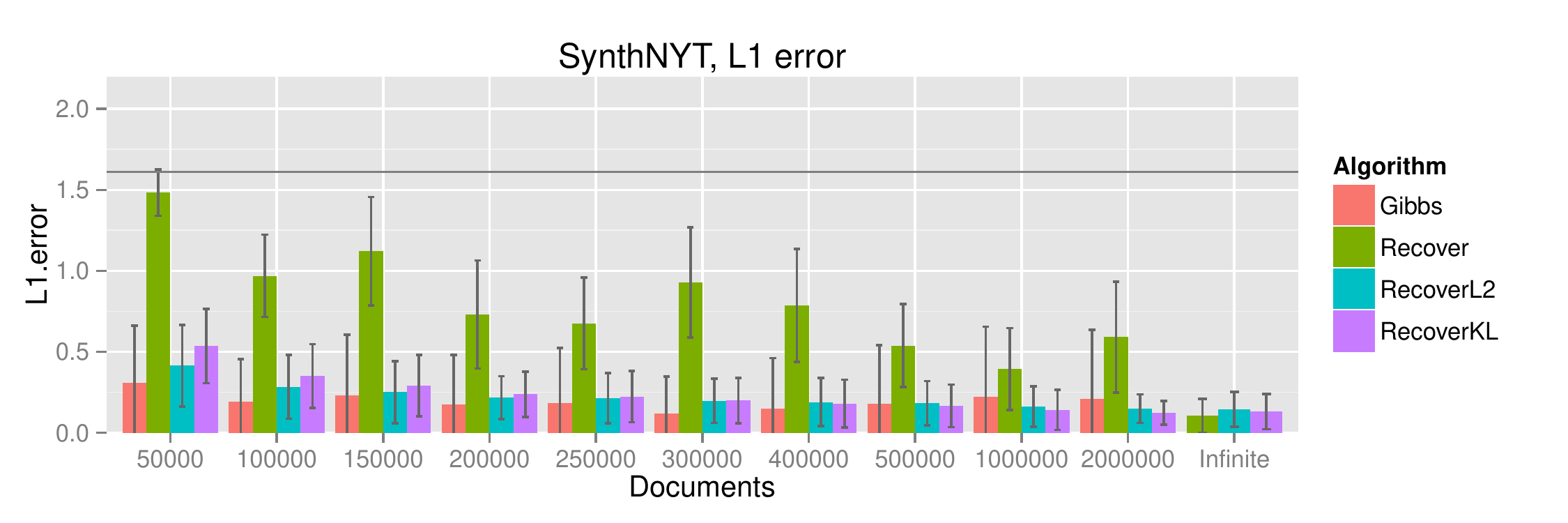}
\caption{$\ell_1$ error for a semi-synthetic model generated from a model trained on NY Times articles with $K=100$. The horizontal line indicates the $\ell_1$ error of $K$ uniform distributions.}
\label{fig:nyt-ss-l1}
\end{center}
\end{figure}

Results for semi-synthetic corpora drawn from NIPS topics are shown in Fig. \ref{fig:nips-ss-l1}.
Recover does poorly for the smallest corpora (topic matching fails for $D=2000$, so $\ell_1$ is not meaningful), but achieves moderate error for $D$ comparable to the NY Times corpus.
RecoverKL and RecoverL2 also do poorly for the smallest corpora, but are comparable to or better than Gibbs sampling, with much lower variance, after 40,000 documents.

\begin{figure}[tbp]
\begin{center}
\includegraphics[scale=0.35]{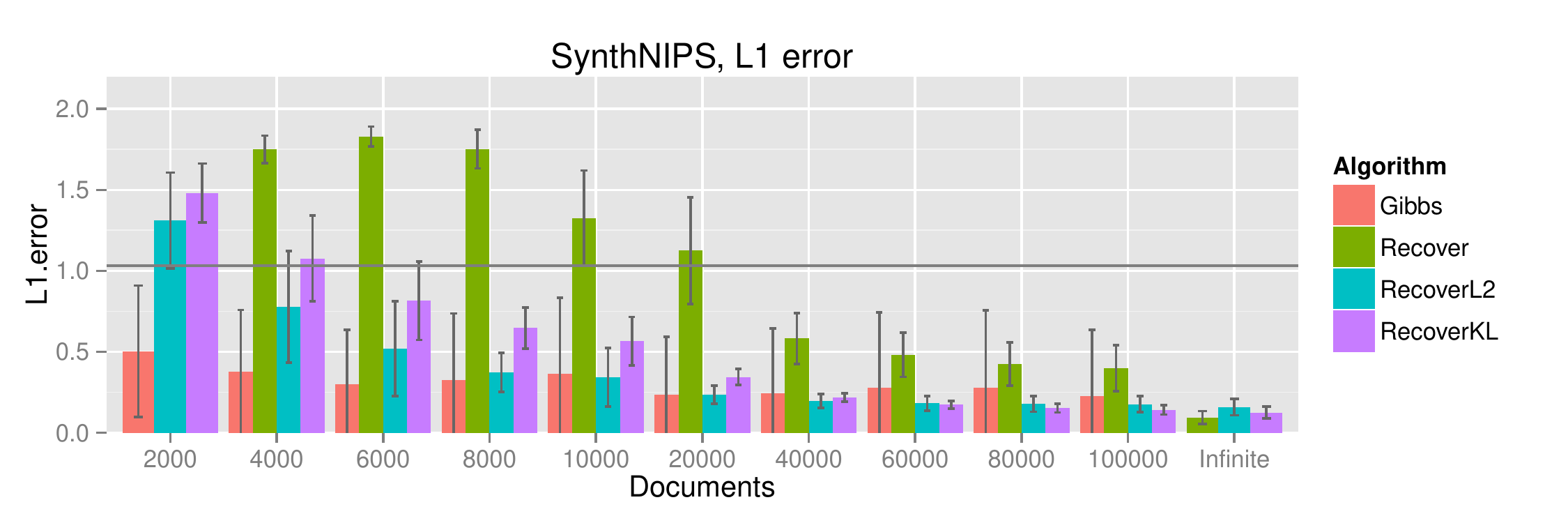}
\caption{$\ell_1$ error for a semi-synthetic model generated from a model trained on NIPS papers with $K=100$. Recover fails for $D=2000$.}
\label{fig:nips-ss-l1}
\end{center}
\end{figure}

\subsection{Effect of separability}

The non-negative algorithms are more robust to violations of the separability assumption than the original Recover algorithm.
In Fig. \ref{fig:nips-ss-l1}, Recover does not achieve zero $\ell_1$ error even with noiseless ``infinite'' data.
Here we show that this is due to lack of separability.
In our semi-synthetic corpora, documents are generated from the LDA model, but the topic-word distributions are learned from data and may not satisfy the anchor words assumption.
We test the sensitivity of algorithms to violations of the separability condition by adding a synthetic anchor word to each topic that is by construction unique to the topic.
We assign the synthetic anchor word a probability equal to the most probable word in the original topic. This causes the distribution to sum to greater than 1.0, so we renormalize.
Results are shown in Fig. \ref{fig:nyt-ss-anchor-l1}.
The $\ell_1$ error goes to zero for Recover, and close to zero for RecoverKL and RecoverL2. The reason RecoverKL and RecoverL2 do not reach exactly zero is because we do not solve the optimization problems to perfect optimality. 

\begin{figure}[tbp]
\begin{center}
\includegraphics[scale=0.35]{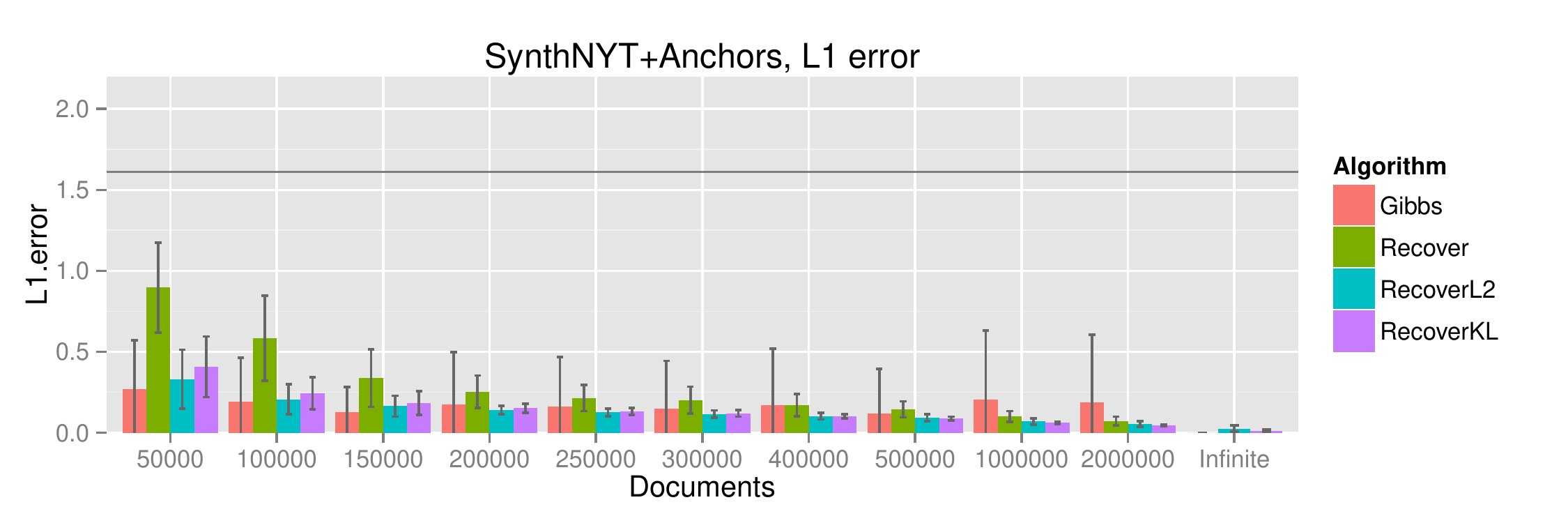}
\caption{When we add artificial anchor words before generating synthetic documents, $\ell_1$ error goes to zero for Recover and close to zero for RecoverKL and RecoverL2.}
\label{fig:nyt-ss-anchor-l1}
\end{center}
\end{figure}

\subsection{Effect of correlation}

The theoretical guarantees of the new algorithms apply even if topics are correlated.
To test how algorithms respond to correlation, we generated new synthetic corpora from the same $K=100$ model trained on NY Times articles. 
Instead of a symmetric Dirichlet distribution, we use a logistic normal distribution with a block-structured covariance matrix.
We partition topics into 10 groups. For each pair of topics in a group, we add a non-zero off-diagonal element to the covariance matrix.
This block structure is not necessarily realistic, but shows the effect of correlation.
Results for two levels of covariance ($\rho = 0.05, \rho = 0.1$) are shown in Fig. \ref{fig:nyt-ss-corr-l1}.
\begin{figure}[t]
\begin{center}
\includegraphics[scale=0.35]{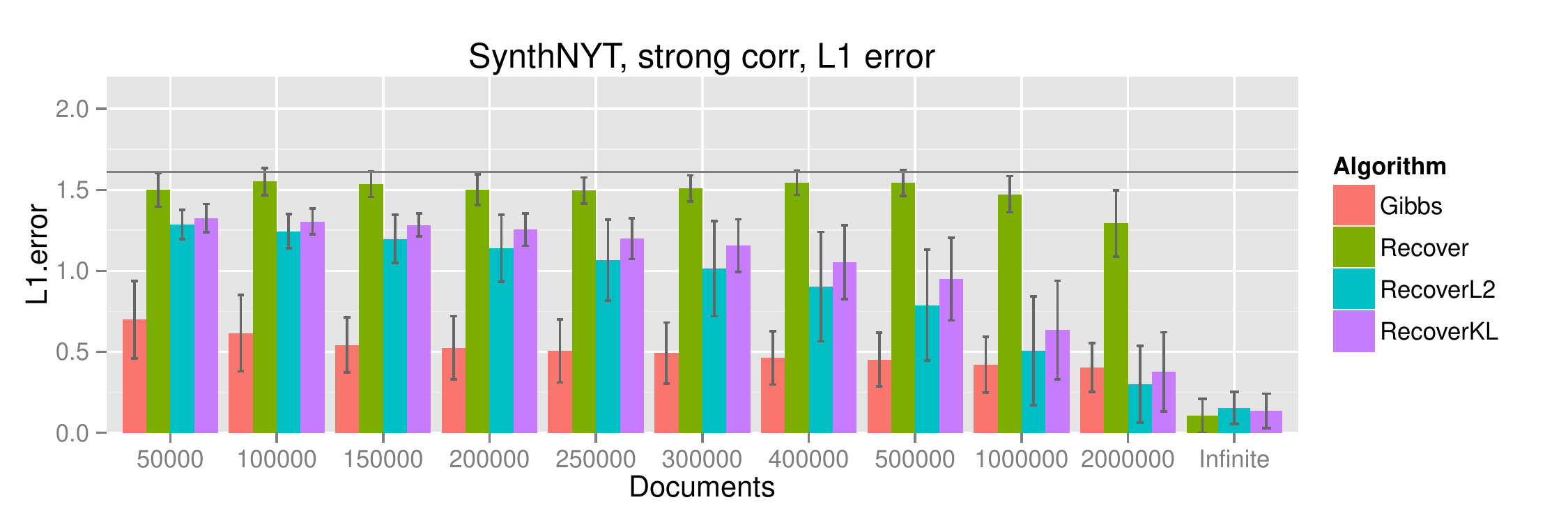}
\includegraphics[scale=0.35]{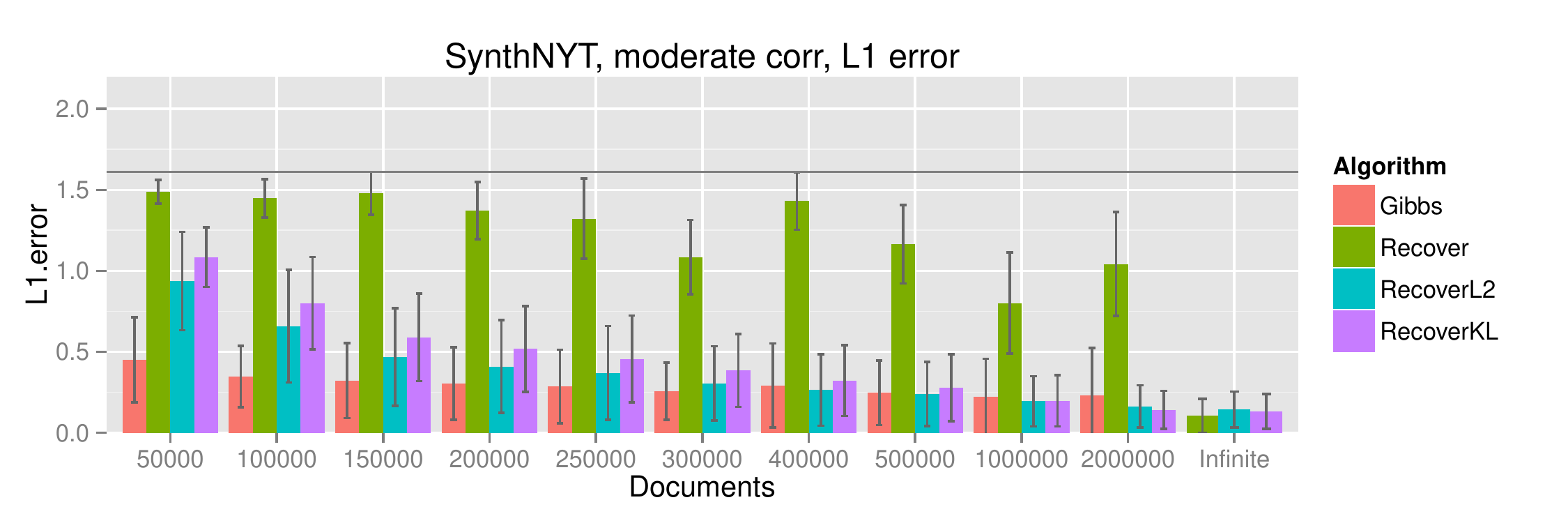}
\caption{$\ell_1$ error increases as we increase topic correlation. We use the same $K=100$ topic model from NY Times articles, but add correlation: TOP $\rho=0.05$, BOTTOM $\rho=0.1$.}
\label{fig:nyt-ss-corr-l1}
\end{center}
\end{figure}
Results for Recover are much worse in both cases than the Dirichlet-generated corpora in Fig. \ref{fig:nyt-ss-l1}.
The other three algorithms, especially Gibbs sampling, are more robust to correlation, but performance consistently degrades as correlation increases, and improves with larger corpora. With infinite data $\ell_1$ error is equal to $\ell_1$ error in the uncorrelated synthetic corpus (non-zero because of violations of the separability assumption). 

\subsection{Real documents}

The new algorithms produce comparable quantitative and qualitative results on real data.
Fig. \ref{fig:real-data} shows three metrics for both corpora.
Error bars show the distribution of log probabilities across held-out {\em documents} (top panel) and coherence and unique words across {\em topics} (center and bottom panels).
Held-out sets are 230 documents for NIPS and 59k for NY Times.
For the small NIPS corpus we average over 5 non-overlapping train/test splits.
The matrix-inversion in Recover failed for the smaller corpus (NIPS).
In the larger corpus (NY Times), Recover produces noticeably worse held-out log probability per token than the other algorithms.
Gibbs sampling produces the best average held-out probability ($p < 0.0001$ under a paired $t$-test), but the difference is within the range of variability between documents.
We tried several methods for estimating hyperparameters, but the observed differences did not change the relative performance of algorithms.
\begin{figure}[t]
\begin{center}
\includegraphics[scale=0.39]{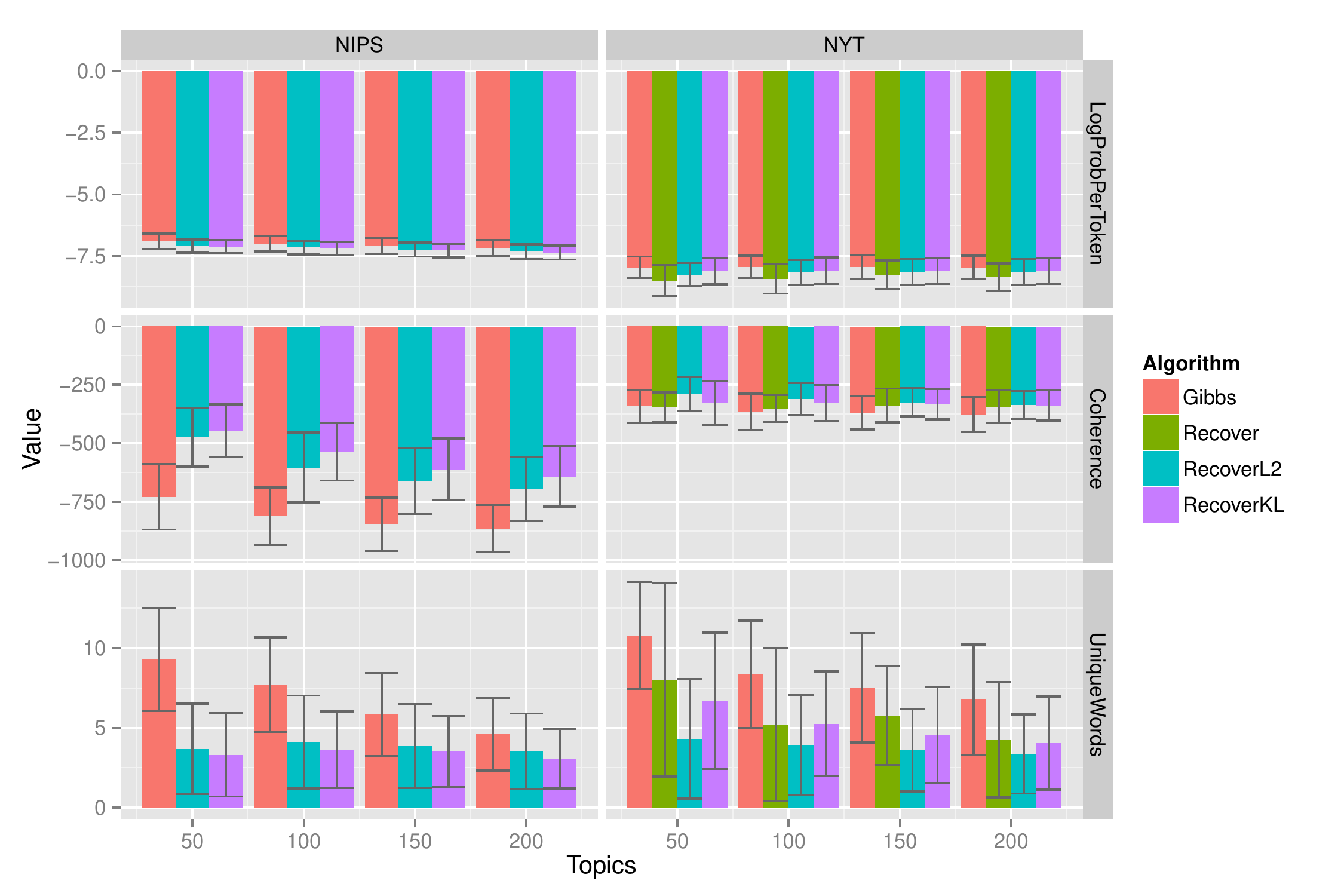}

\caption{Held-out probability (per token) is similar for RecoverKL, RecoverL2, and Gibbs sampling. RecoverKL and RecoverL2 have better coherence, but fewer unique words than Gibbs. (Up is better for all three metrics.)}
\label{fig:real-data}
\end{center}
\end{figure}
Gibbs sampling has worse coherence than the Recover algorithms, but produces more unique words per topic.
These patterns are consistent with semi-synthetic results for similarly sized corpora (details are in supplementary material).

For each NY Times topic learned by RecoverL2 we find the closest Gibbs topic by $\ell_1$ distance.
The closest, median, and farthest topic pairs are shown in Table \ref{tbl:samples}.\footnote{The UCI NY Times corpus includes named-entity annotations, indicated by the {\em zzz} prefix.}
We observe that when there is a difference, recover-based topics tend to have more specific words ({\em Anaheim Angels} vs. {\em pitch}).

\begin{table}[tdp]

\caption{Example topic pairs from NY Times (closest $\ell_1$), anchor words in bold. All 100 topics are in suppl. material.}
\begin{center}
\rowcolors{1}{}{gray!20}
{\tiny
\begin{tabular}{c|p{6cm}}
RecoverL2 & run inning game hit season zzz\_anaheim\_angel \\ 
Gibbs & run inning hit game ball pitch \\ 
\hline
RecoverL2 &  father family {\bf zzz\_elian} boy court zzz\_miami  \\
Gibbs & zzz\_cuba zzz\_miami cuban zzz\_elian boy protest  \\
\hline
RecoverL2 & {\bf file} sport read internet email zzz\_los\_angeles \\ 
Gibbs & web site com www mail zzz\_internet \\ 
\end{tabular}
}
\end{center}
\label{tbl:samples}

\end{table}%

\section{Conclusions}

We present new algorithms for topic modeling, inspired by \citet{AGM}, which are efficient and simple to implement yet maintain provable guarantees.
The running time of these algorithms is effectively independent of the size of the corpus.
Empirical results suggest that the sample complexity of these algorithms is somewhat greater than MCMC, but, particularly for the $\ell_2$ variant, they provide comparable results in a fraction of the time.
We have tried to use the output of our algorithms as initialization for further optimization (e.g. using MCMC) but have not yet found a hybrid that out-performs either method by itself.
Finally, although we defer parallel implementations to future work, these algorithms are parallelizable, potentially supporting web-scale topic inference.

\appendix

\newpage

\section{Proof for Anchor-Words Finding Algorithm}

Recall that the correctness of the algorithm depends on the following Lemmas:

\begin{lemma}
\label{lem:perturbation}
There is a vertex $v_i$ whose distance from $\spn(S)$ is at least $\gamma/2$.
\end{lemma}


\begin{lemma}
\label{lem:inductionforfastalg}
The point $\Delta_j$ found by the algorithm must be $\delta = O(\epsilon/\gamma^2)$ close to some vertex $v_i$.
\end{lemma}



In order to prove Lemma~\ref{lem:perturbation}, we use a volume argument. First we show that the volume of a robust simplex cannot change by too much when the vertices are perturbed.

\begin{lemma}
\label{lem:volperturbation}
Suppose $\{v_1, v_2, ..., v_K\}$ are the vertices of a $\gamma$-robust simplex $S$. Let $S'$ be a simplex with vertices $\{v_1', v_2', ..., v_K'\}$, each of the vertices $v_i'$ is a perturbation of $v_i$ and $\norm{v_i' - v_i}_2 \le \delta$. When $10\sqrt{K} \delta < \gamma$ the volume of the two simplices satisfy

\[
\mbox{vol}(S) (1-2\delta/\gamma)^{K-1} \le \mbox{vol}(S') \le \mbox{vol}(S) (1+4\delta/\gamma)^{K-1}.
\]
\end{lemma}

\begin{proof}
As the volume of a simplex is proportional to the determinant of a matrix whose columns are the edges of the simplex, we first
show the following perturbation bound for determinant.

\begin{claim}
Let $A$, $E$ be $K\times K$ matrices, the smallest eigenvalue of $A$ is at least $\gamma$, the Frobenius norm $\norm{E}_F \le \sqrt{K}\delta$, when $\gamma > 5 \sqrt{K}\delta$ we have

\[
\det(A+E)/\det(A) \ge (1-\delta/\gamma)^K.
\]
\end{claim}

\begin{proof}
Since $\det(AB) = \det(A)\det(B)$, we can multiply both $A$ and $A+E$ by $A^{-1}$. Hence $\det(A+E)/\det(A) = \det(I + A^{-1}E)$. 

The Frobenius norm of  $A^{-1}E$ is bounded by

$$\norm{A^{-1}E}_F \le \norm{A^{-1}}_2 \norm{E}_F \le \sqrt{K} \delta/\gamma.$$

Let the eigenvalues of $A^{-1}E$ be $\lambda_1, \lambda_2, ..., \lambda_K$, then by definition of Frobenius Norm 
$\sum_{i=1}^K \lambda_i^2 \le \norm{A^{-1}E}_F^2 \le K\delta^2/\gamma^2$.
The eigenvalues of $I + A^{-1}E$ are just $1+\lambda_1, 1+\lambda_2, ..., 1+\lambda_K$, and the
determinant $\det(I+A^{-1}E) = \prod_{i=1}^K (1+\lambda_i)$. Hence it suffices to show

$$\min \prod_{i=1}^K (1+\lambda_i) \ge (1-\delta/\gamma)^K \mbox{ when }\sum_{i=1}^K \lambda_i^2 \le K\delta^2/\gamma^2.$$

To do this we apply Lagrangian method and show the minimum is only obtained when all $\lambda_i$'s are equal. The optimal value must be obtained at a local optimum of 

$$
\prod_{i=1}^K (1+\lambda_i)  + C \sum_{i=1}^K \lambda_i^2.
$$

Taking partial derivatives with respect to $\lambda_i$'s, we get the equations $ -\lambda_i (1+\lambda_i) = -\prod_{i=1}^K (1+\lambda_i) / 2C$ (here using $\sqrt{K}\delta/\gamma$ is small so $1+\lambda_i > 1/2 > 0$). The right hand side is a constant, so each $\lambda_i$ must be one of the two solutions of this equation. However, only one of the solution is larger than $1/2$, therefore all the $\lambda_i$'s are equal.
\end{proof}

For the lower bound, we can project the perturbed subspace to the $K-1$ dimensional space. Such 
a projection cannot increase the volume and the perturbation distances only get smaller. 
Therefore we can apply the claim directly, the columns of $A$ are just $v_{i+1}- v_1$ for $i = 1, 2, ..., K-1$; columns of $E$ are just $v_{i+1}'-v_{i+1} - (v_1'-v_1)$. The smallest eigenvalue of $A$ is at least $\gamma$ because the polytope is $\gamma$ robust, which is equivalent to saying after orthogonalization each column still has length at least $\gamma$. The Frobenius norm of $E$ is at most $2\sqrt{K-1}\delta$. We get the lower bound directly by applying the claim.

For the upper bound, swap the two sets $S$ and $S'$ and use the argument for the lower bound. The only thing we need to show is that the smallest eigenvalue of the matrix generated by points in $S'$ is still at least $\gamma/2$. This follows from Wedin's Theorem \citep{wedin1972perturbation}
and the fact that $\norm{E} \le \norm{E}_F \le \sqrt{K}\delta \le \gamma/2$.
\end{proof}

\begin{figure}
\begin{center}
\includegraphics[width=3.5in]{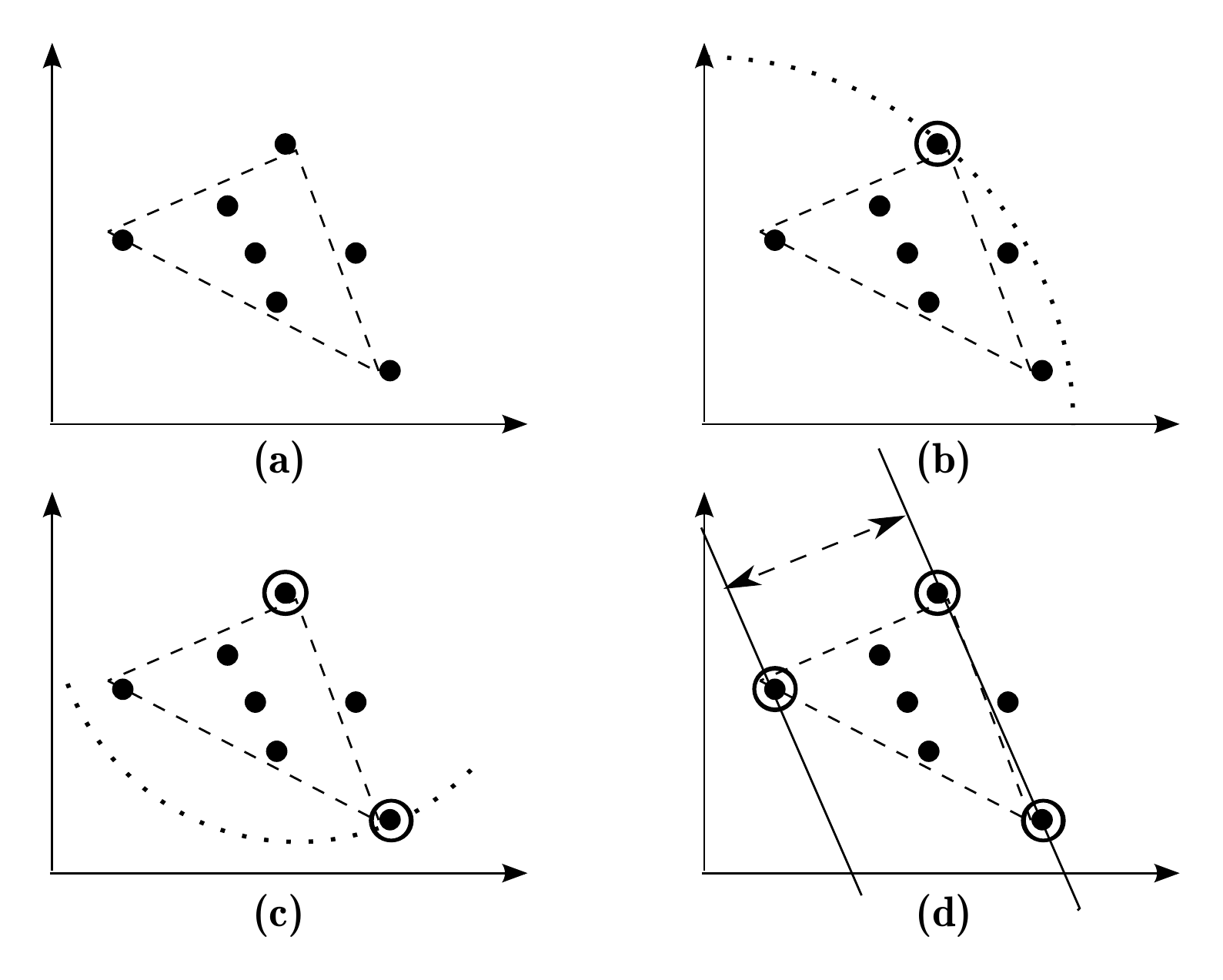}
\caption{Illustration of the Algorithm}
\end{center}
\end{figure}

Now we are ready to prove Lemma~\ref{lem:perturbation}.

\begin{proof} 
The first case is for the first step of the algorithm, when we try to find the farthest point to the origin. Here
essentially $S = \{\vec{0}\}$. For any two vertices $v_1, v_2$, since the simplex is $\gamma$ robust, the distance between
$v_1$ and $v_2$ is at least $\gamma$. Which means $\dis(\vec{0},v_1)+\dis(\vec{0},v_2) \ge \gamma$, one of them
must be at least $\gamma/2$.

For the later steps, recall that $S$ contains vertices of a perturbed simplex. Let $S'$ be the set of original vertices corresponding to the perturbed vertices in $S$. Let $v$ be any vertex in $\{v_1, v_2, ..., v_K\}$ which is not in $S$. Now we know the distance between $v$ and $S$ is equal to $\mbox{vol}(S\cup\{v\})/ (|S|-1) \mbox{vol}(S)$.  On the other hand, we know $\mbox{vol}(S'\cup\{v\})/ (|S'|-1) \mbox{vol}(S') \ge \gamma$. Using Lemma~\ref{lem:volperturbation} to bound the ratio between the two pairs $\mbox{vol}(S)/\mbox{vol}(S')$ and $\mbox{vol}(S\cup\{v\})/\mbox{vol}(S'\cup\{v\})$, we get:
$$\dis(v,S) \ge (1-4\epsilon'/\gamma)^{2|S|-2} \gamma > \gamma/2$$
when $\gamma > 20K \epsilon'$.
\end{proof}

Lemma~\ref{lem:inductionforfastalg} is based on the following observation: in a simplex the point with largest $\ell_2$ is always a vertex. Even if two vertices have the same norm if they are not close to each other the vertices on the edge connecting them will have significantly lower norm.

\begin{proof} (Lemma~\ref{lem:inductionforfastalg}) 

Since $d_j$ is the point found by the algorithm, let us consider the point $a_j$ before perturbation. The point $a_j$ is inside the simplex, therefore we can write $a_j$ as a convex combination of the vertices:

$$ a_j = \sum_{t = 1}^K c_t v_t$$

Let $v_t$ be the vertex with largest coefficient $c_t$. Let $\Delta$ be the largest distance from some vertex to the space spanned by points in $S$ ($\Delta = \max_l \dis(v_l, \spn(S))$. By Lemma~\ref{lem:perturbation} we know $\Delta > \gamma/2$. Also notice that we are not assuming $\dis(v_t, \spn(S)) = \Delta$.

Now we rewrite $a_j$ as $c_t v_t + (1-c_t) w$, where $w$ is a vector in the convex hull of vertices other than $v_t$. 
Observe that $a_j$ must be far from $\spn(S)$, because $d_j$ is the farthest point found by the algorithm. Indeed:
$$ \dis(a_j, \spn(S))  \ge \dis(d_j,\spn(S)) - \epsilon \ge \dis(v_l, \spn(S)) - 2\epsilon \ge \Delta-2\epsilon$$

The second inequality is because there must be some point $d_l$ that correspond to the farthest vertex $v_l$ and have $\dis(d_l, \spn(S)) \ge \Delta-\epsilon$. Thus as $d_j$ is the farthest point $\dis(d_j, \spn(S)) \ge \dis(d_l, \spn(S)) \ge \Delta-\epsilon$. 

The point $a_j$ is on the segment connecting $v_t$ and $w$, the distance between $a_j$ and $\spn(S)$ is not much smaller than that of $v_t$ and $w$. Following the intuition in $\ell_2$ norm when $v_t$ and $w$ are far we would expect $a_j$ to be very close to either $v_t$ or $w$. Since $c_t \ge 1/K$ it cannot be really close to $w$, so it must be really close to $v_t$. We formalize this intuition by the following calculation (see Figure~\ref{fig:twodimensional}):

\begin{figure}
\begin{center}
\includegraphics[width=3.5in]{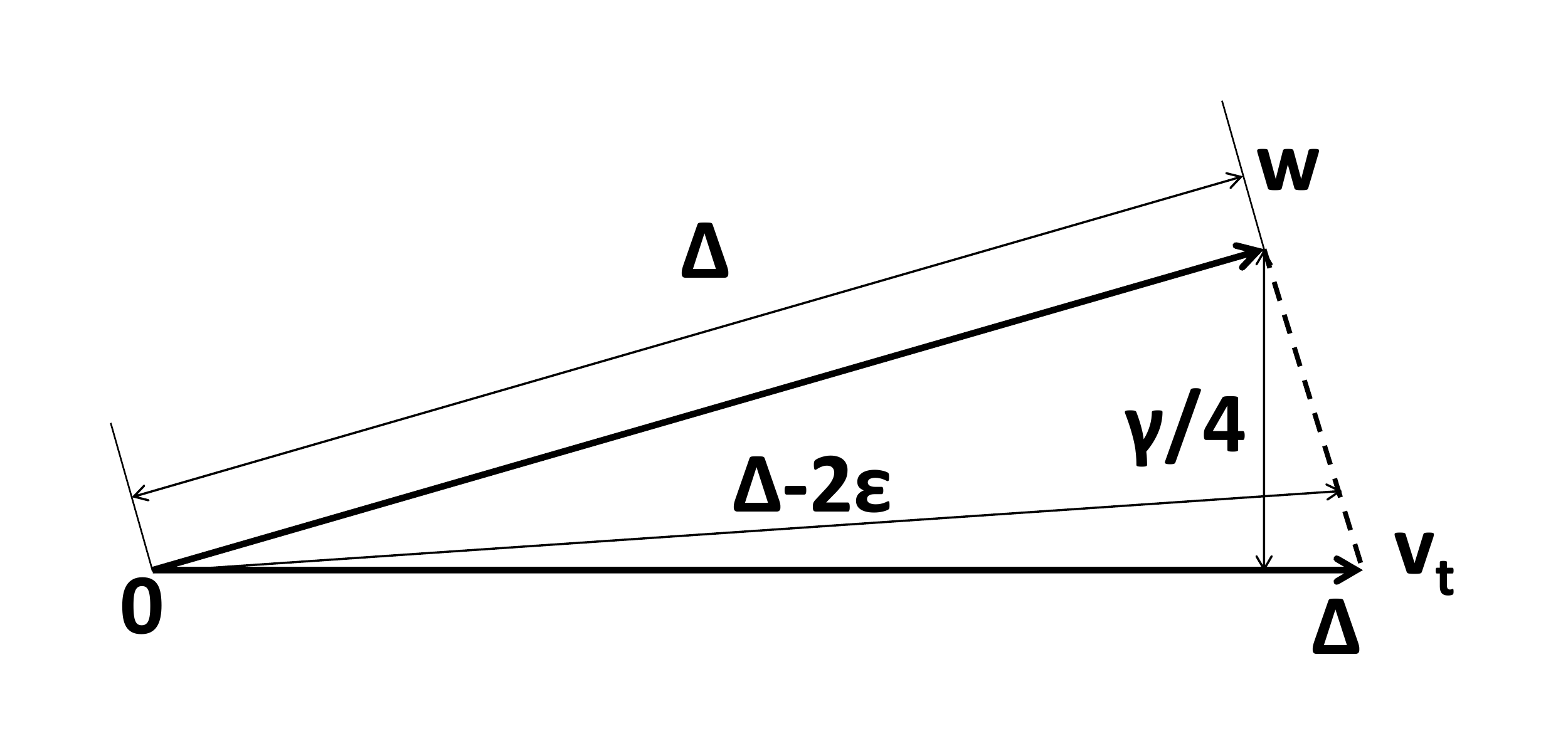}
\caption{Proof of Lemma~\ref{lem:inductionforfastalg}, after projecting to the orthogonal subspace of $\spn(S)$.}
\label{fig:twodimensional}
\end{center}
\end{figure}

Project everything to the orthogonal subspace of $\spn(S)$ (points in $\spn(S)$ are now at the origin). After projection distance to $\spn(S)$ is just the $\ell_2$ norm of a vector. Without loss of generality we assume $\norm{v_t}_2 = \norm{w}_2 = \Delta$ because these two have length at most $\Delta$, and extending these two vectors to have length $\Delta$ can only increase the length of $d_j$. 

The point $v_t$ must be far from $w$ by applying Lemma~\ref{lem:perturbation}: consider the set of vertices $V' = \{v_i: v_i\mbox{ does not correspond to any point in $S$ and } i\ne t\}$. The set $V'\cup S$ satisfy the assumptions in Lemma~\ref{lem:perturbation} so there must be one vertex that is far from $\spn(V'\cup S)$, and it can only be $v_t$. Therefore even after projecting to orthogonal subspace of $\spn(S)$, $v_t$ is still far from any convex combination of $V'$. The vertices that are not in $V'$ all have very small norm after projecting to orthogonal subspace (at most $\delta_0$) so we know the distance of $v_t$ and $w$ is at least $\gamma/2-\delta_0 > \gamma/4$.

Now the problem becomes a two dimensional calculation. When $c_t$ is fixed the length of $a_j$ is strictly increasing when the distance of $v_t$ and $w$ decrease, so we assume the distance is $\gamma/4$. Simple calculation (using essentially just pythagorean theorem) shows

$$ c_t(1-c_t) \le \frac{\epsilon}{\Delta-\sqrt{\Delta^2-\gamma^2/16}}. $$

The right hand side is largest when $\Delta = 2$ (since the vectors are in unit ball) and the maximum value is $O(\epsilon/\gamma^2)$. When this value is smaller than $1/K$, we must have $1-c_t \le O(\epsilon/\gamma^2)$. Thus $c_t \ge 1-O(\epsilon/\gamma^2)$ and $\delta \le (1-c_t) + \epsilon \le O(\epsilon/\gamma^2)$. 
\end{proof}

The cleanup phase tries to find the farthest point to a subset of $K-1$ vertices, and use that point as the $K$-th vertex. This will improve the result because when we have $K-1$ points close to $K-1$ vertices, only one of the vertices can be far from their span. Therefore the farthest point must be close to the only remaining vertex. Another way of viewing this is that the algorithm is trying to greedily maximize the volume of the simplex, which makes sense because the larger the volume is, the more words/documents the final LDA model can explain.

The following lemma makes the intuitions rigorous and shows how cleanup improves the guarantee of Lemma~\ref{lem:inductionforfastalg}.

\begin{lemma}\label{lem:cleanup}
Suppose $|S| = K-1$ and each point in $S$ is $\delta = O(\epsilon/\gamma^2) < \gamma/20K$ close to some vertex $v_i$, then the farthest point $v_j'$ found by the algorithm is $1 - O(\epsilon/\gamma)$ close to the remaining vertex.
\end{lemma}


\begin{proof}
We still look at the original point $a_j$ and express it as $\sum_{t=1}^K c_tv_t$. Without loss of generality let $v_1$ be the vertex that does not correspond to anything in $S$. By Lemma~\ref{lem:perturbation} $v_1$ is $\gamma/2$ far from $\spn(S)$. On the other hand all other vertices are at least $\gamma/20r$ close to $\spn(S)$. We know the distance $\dis(a_j, \spn(S)) \ge \dis(v_1, \spn(S)) - 2\epsilon$, this cannot be true unless $c_t \ge 1 - O(\epsilon/\gamma)$.
\end{proof}

These lemmas directly lead to the following theorem:

\begin{theorem}
\label{thm:fastanchorwords}
FastAnchorWords algorithm runs in time $\tilde{O}(V^2+VK/\epsilon^2)$ and outputs a subset of $\{d_1,...,d_V\}$ of size $K$ that $O(\epsilon/\gamma)$-covers the vertices provided that  $20 K \epsilon/\gamma^2 < \gamma$.
\end{theorem}

\begin{proof}
In the first phase of the algorithm, do induction using Lemma~\ref{lem:inductionforfastalg}. When $20 K \epsilon/\gamma^2 < \gamma$ Lemma~\ref{lem:inductionforfastalg} shows that we find a set of points that $O(\epsilon/\gamma^2)$-covers
the vertices. Now Lemma~\ref{lem:cleanup} shows after cleanup phase the points are refined to $O(\epsilon/\gamma)$-cover the vertices. 
\end{proof}

\section{Proof for Nonnegative Recover Procedure}

In order to show RecoverL2 learns the parameters even when the rows of $\bar{Q}$ are perturbed, we need the
following lemma that shows when columns of $\bar{Q}$ are close to the expectation, the posteriors $c$ computed by the algorithm is
also close to the true value.

\begin{lemma}
\label{lem:recoverposterior}
For a $\gamma$ robust simplex $S$ with vertices $\{v_1, v_2, ..., v_K\}$, let $v$ be a point in the simplex that can be represented
as a convex combination $v = \sum_{i=1}^K c_i v_i$. If the vertices of $S$ are perturbed to $S' = \{..., v_i', ...\}$ where
$\norm{v_i' - v_i} \le \delta_1$ and $v$ is perturbed to $v'$ where $\norm{v-v'} \le \delta_2$. Let $v^*$
be the point in $S'$ that is closest to $v'$, and $v^* = \sum_{i=1}^K c_i' v_i$, when $10\sqrt{K}\delta_1\le \gamma$ for all $i\in[K]$
$|c_i-c_i'| \le 4(\delta_1+\delta_2)/\gamma$.
\end{lemma}

\begin{proof}
Consider the point $u = \sum_{i=1}^K c_i v_i'$, by triangle inequality: $\norm{u-v} \le \sum_{i=1}^K c_i \norm{v_i-v_i'} \le \delta_1$. Hence $\norm{u-v'} \le \norm{u-v}+\norm{v-v'} \le \delta_1+\delta_2$, and $u$ is in $S'$. The point $v^*$ is the 
point in $S'$ that is closest to $v'$, so $\norm{v^*-v'} \le \delta_1+\delta_2$ and $\norm{v^*-u} \le 2(\delta_1+\delta_2)$.

Then we need to show when a point ($u$) moves a small distance, its representation also changes by a small amount. Intuitively
this is true because $S$ is $\gamma$ robust. By Lemma~\ref{lem:perturbation} when $10\sqrt{K}\delta_1 < \gamma$, the simplex
$S'$ is also $\gamma/2$ robust. For any $i$, let $Proj_i(v^*)$ and $Proj_i(u)$ be the projections of $v^*$ and $u$ in the orthogonal
subspace of $\spn(S'\backslash v_i')$, then $$|c_i-c_i'| =\norm{Proj_i(v^*) - Proj_i(u)}/\dis(v_i,\spn(S'\backslash v_i'))
\le 4(\delta_1+\delta_2)/\gamma$$ and this completes the proof.
\end{proof}

With this lemma it is not hard to show that RecoverL2 has polynomial sample complexity.

\begin{theorem}
When the number of documents $M$ is at least $$\max\{O(aK^3\log V/D(\gamma p)^6\epsilon),O((aK)^3\log V/D\epsilon^3(\gamma p)^4)\}$$ our algorithm using the conjunction of FastAnchorWords and RecoverL2 learns the $A$ matrix with entry-wise error at most $\epsilon$.
\end{theorem}

\begin{proof} (sketch)
We can assume without loss of generality that
each word occurs with probability at least $\epsilon/4aK$  and furthermore that if $M$ is
at least $50\log V/D\epsilon_Q^2$ then the empirical matrix $\tilde{Q}$ is entry-wise within an additive
$\epsilon_Q$ to the true $Q = \frac{1}{M} \sum_{d=1}^M A W_d W_d^TA^T$ see \citep{AGM} for the details.
Also, the $K$ anchor rows of $\bar{Q}$ form a simplex that is $\gamma p$ robust.

The error in each column of $\bar{Q}$ can be at most $\delta_2 = \epsilon_Q \sqrt{4aK/\epsilon}$. By Theorem~\ref{thm:fastanchorwords} when $20K\delta_2/(\gamma p)^2 < \gamma p$
(which is satisfied when $M = O(aK^3\log V/D(\gamma p)^6\epsilon)$)
, the anchor words found
are $\delta_1 = O(\delta_2/(\gamma p))$ close to the true anchor words. Hence by Lemma~\ref{lem:recoverposterior} every entry of $C$ has error at most $O(\delta_2/(\gamma p)^2)$.

With such number of documents, all the word probabilities $p(w = i)$ are estimated more accurately than the entries of
$C_{i,j}$, so we omit their perturbations here for simplicity. When we apply the Bayes rule, we know $A_{i,k} = C_{i,k}p(w=i)/p(z=k)$, where $p(z=k)$ is $\alpha_k$ which is lower bounded by $1/aK$. The numerator and denominator are all related to entries of $C$ with positive coefficients sum up to at most 1. Therefore the errors $\delta_{num}$ and $\delta_{denom}$ are at most the error of a single entry of $C$, which is bounded by $O(\delta_2/(\gamma p)^2)$. Applying Taylor's Expansion to $(p(z=k,w=i)+\delta_{num})/ (\alpha_k + \delta_{denom})$, the error on entries of $A$ is at most 
$O(aK\delta_2/(\gamma p)^2)$. When $\epsilon_Q \le O((\gamma p)^2 \eps^{1.5}/(aK)^{1.5})$, we have $O(aK\delta_2/(\gamma p)^2) \le \epsilon$, and get the desired accuracy of $A$. The number of document required is
$M = O((aK)^3\log V/D\epsilon^3(\gamma p)^4)$.

The sample complexity for $R$ can then be bounded using matrix perturbation theory.
\end{proof}

\section{Empirical Results}

This section contains plots for $\ell_1$, held-out probability, coherence, and uniqueness for all semi-synthetic data sets.
Up is better for all metrics except $\ell_1$ error.

\begin{figure}[h]
\begin{center}
\includegraphics[scale=0.35]{nytl1_error.pdf}
\includegraphics[scale=0.35]{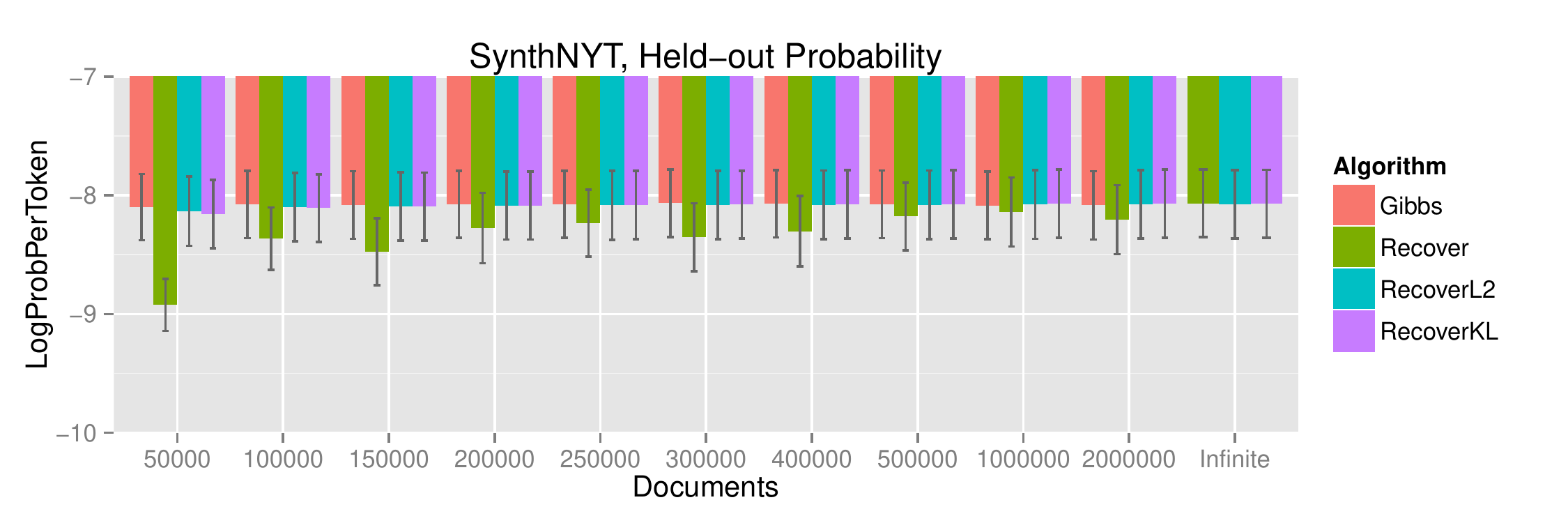}
\includegraphics[scale=0.35]{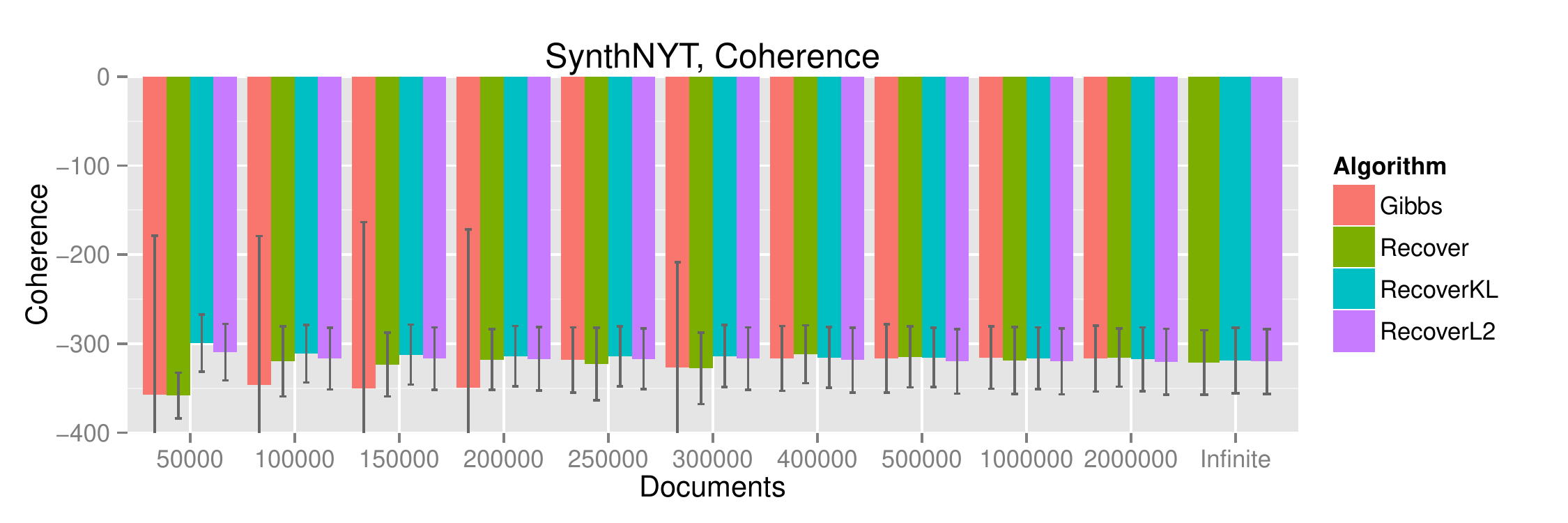}
\includegraphics[scale=0.35]{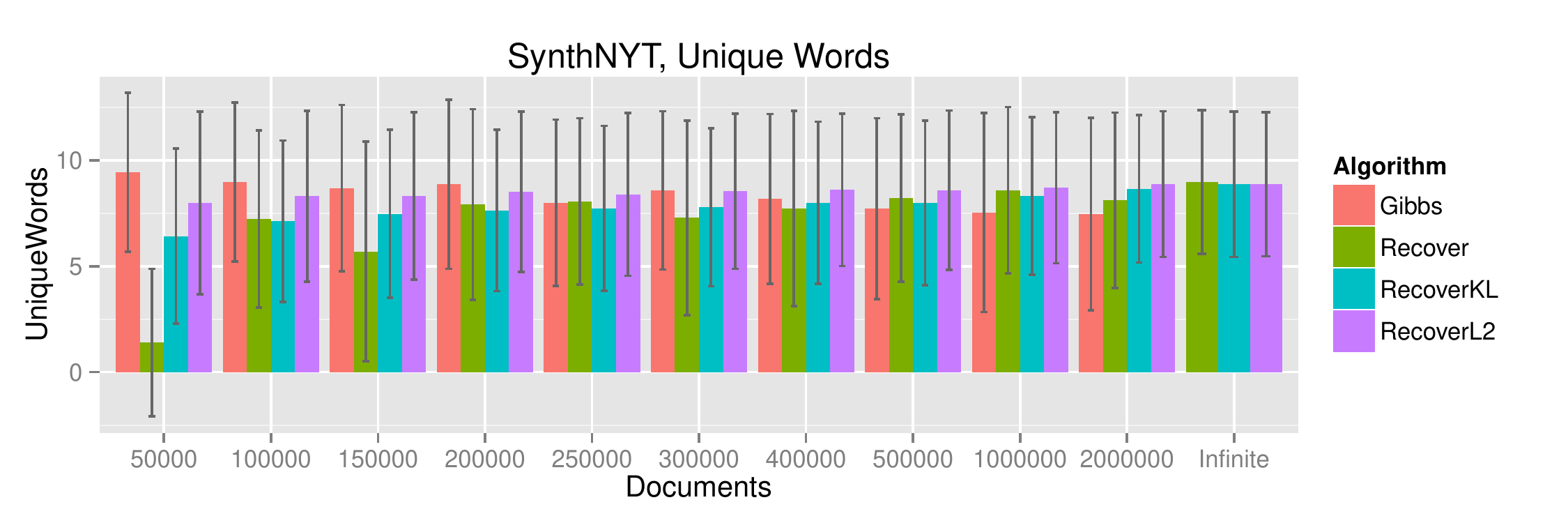}
\caption{Results for a semi-synthetic model generated from a model trained on NY Times articles with $K=100$.}
\end{center}
\end{figure}

\begin{figure}[h]
\begin{center}
\includegraphics[scale=0.35]{addedl1_error.pdf}
\includegraphics[scale=0.35]{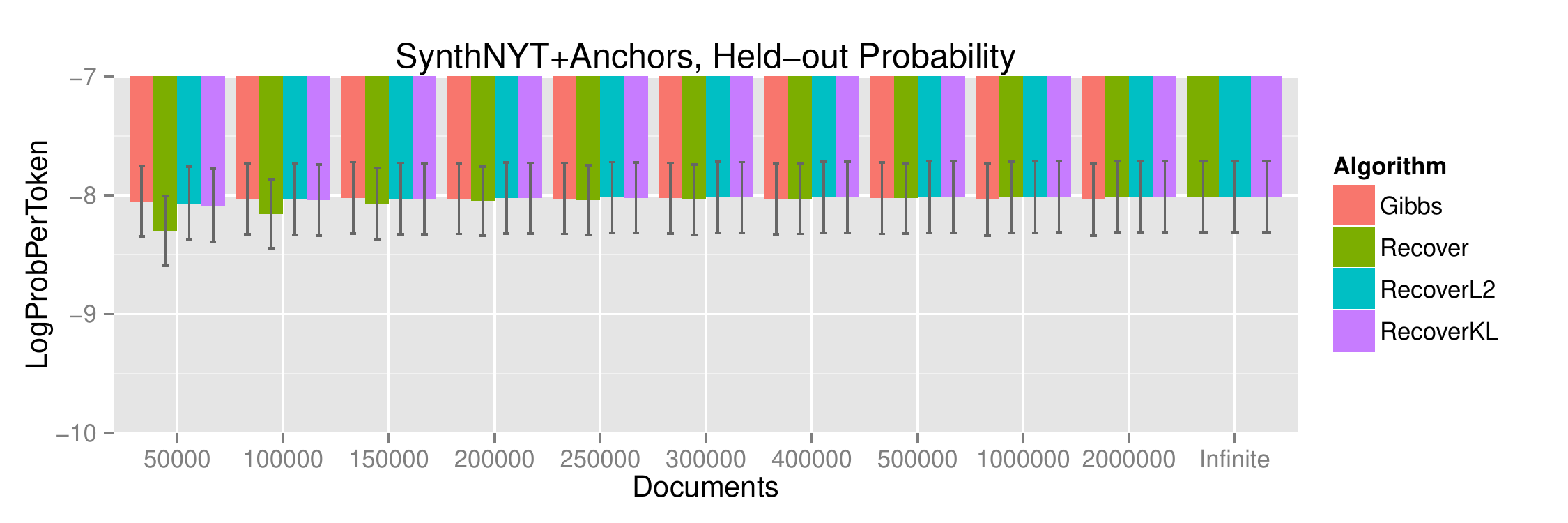}
\includegraphics[scale=0.35]{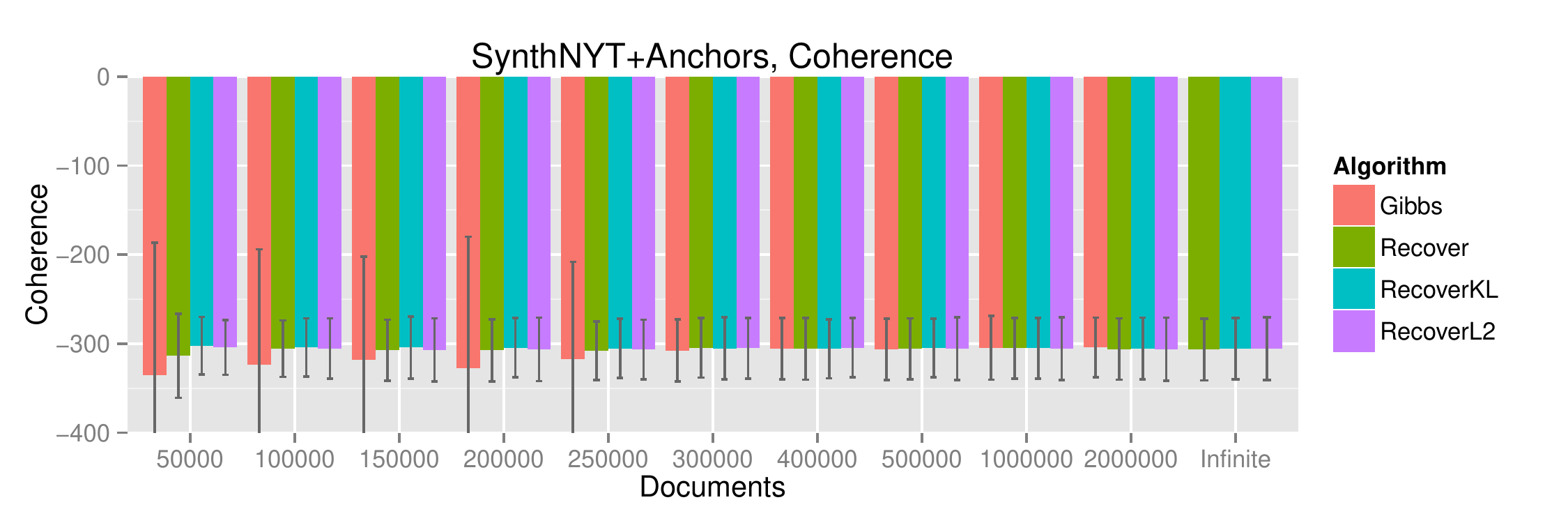}
\includegraphics[scale=0.35]{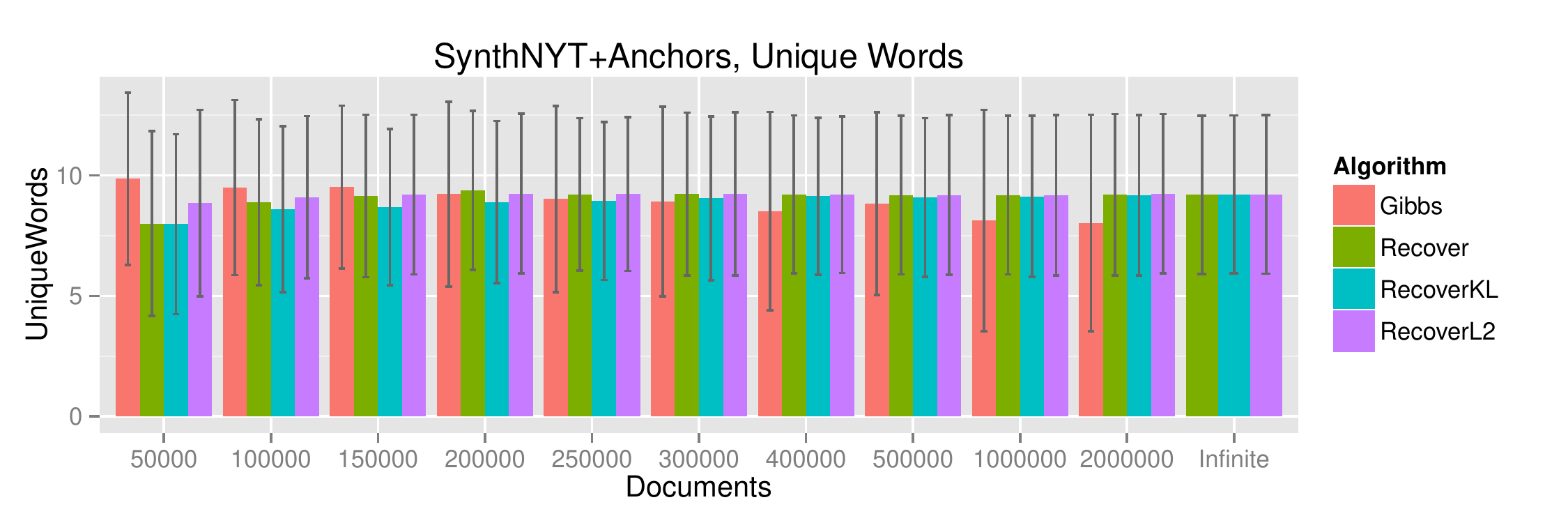}
\caption{Results for a semi-synthetic model generated from a model trained on NY Times articles with $K=100$, with a synthetic anchor word added to each topic.}
\end{center}
\end{figure}

\begin{figure}[h]
\begin{center}
\includegraphics[scale=0.35]{corr4l1_error.pdf}
\includegraphics[scale=0.35]{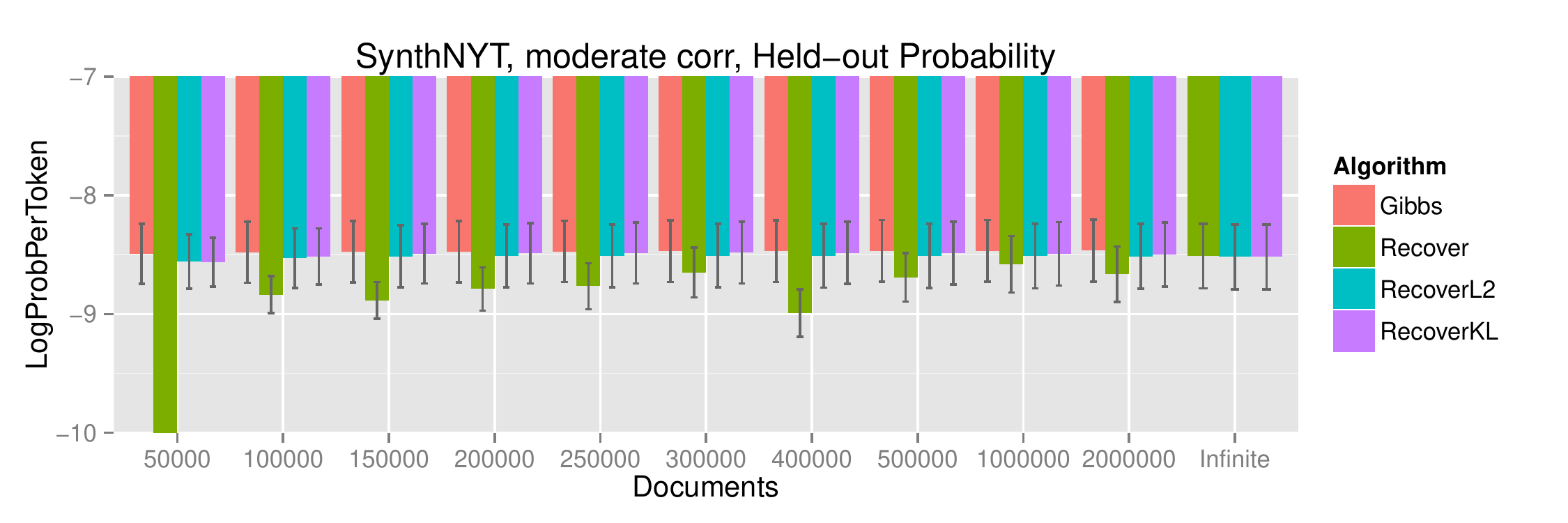}
\includegraphics[scale=0.35]{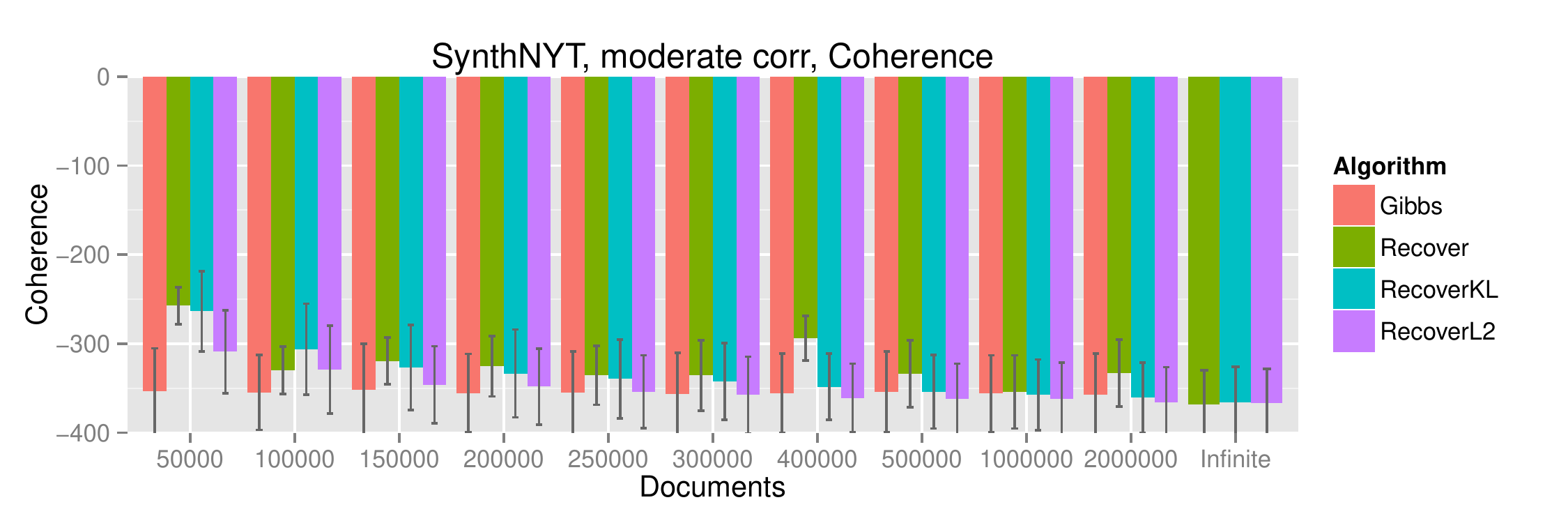}
\includegraphics[scale=0.35]{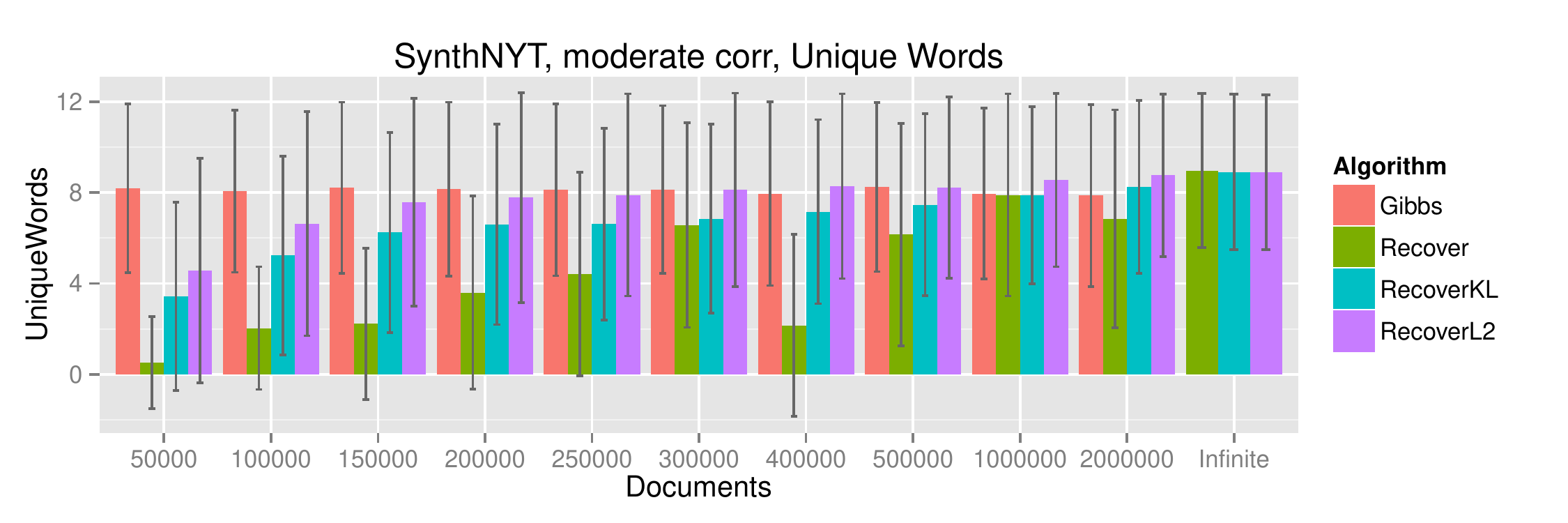}
\caption{Results for a semi-synthetic model generated from a model trained on NY Times articles with $K=100$, with moderate correlation between topics.}
\end{center}
\end{figure}

\begin{figure}[h]
\begin{center}
\includegraphics[scale=0.35]{corr2l1_error.pdf}
\includegraphics[scale=0.35]{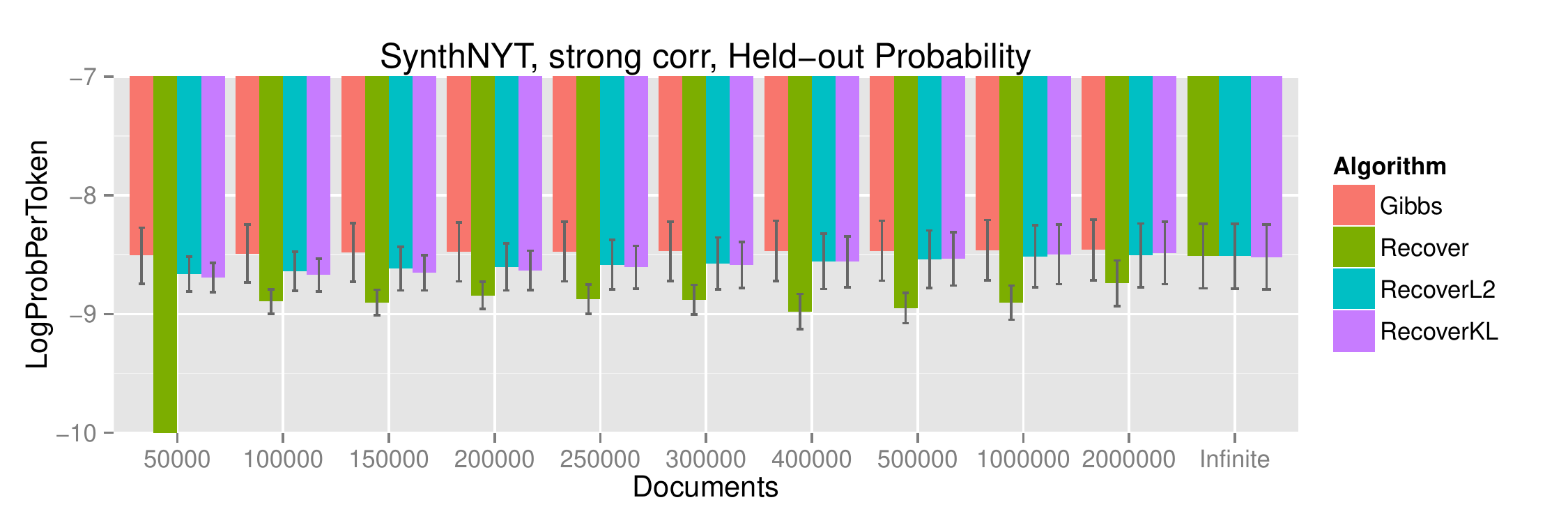}
\includegraphics[scale=0.35]{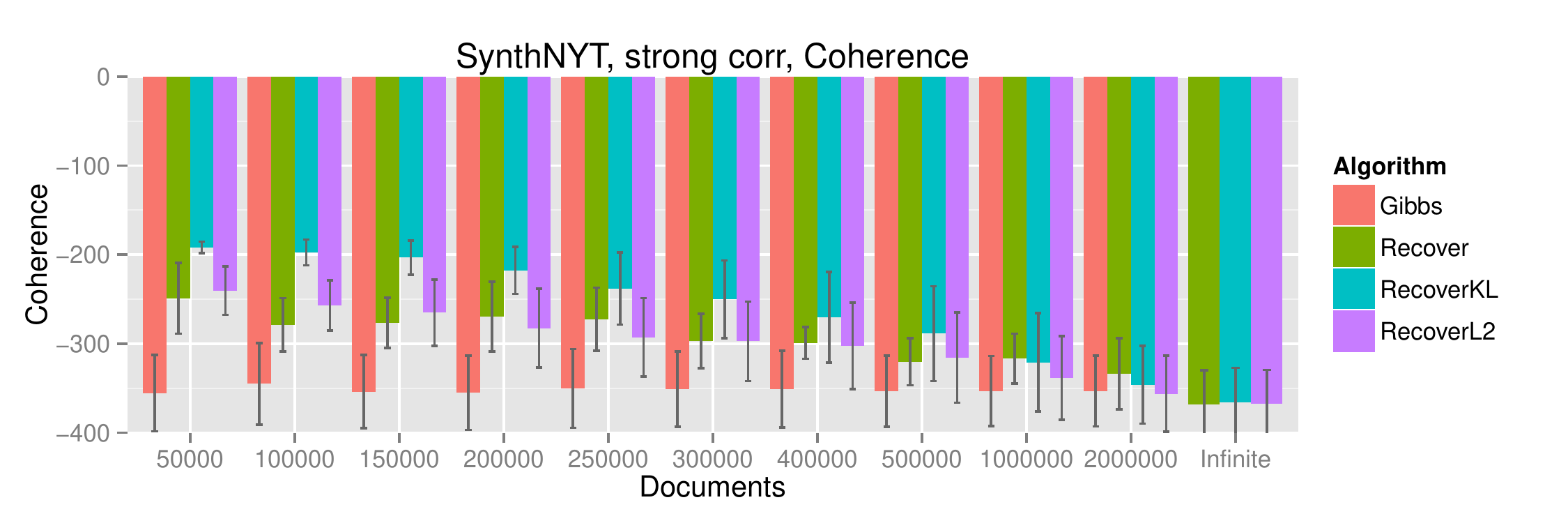}
\includegraphics[scale=0.35]{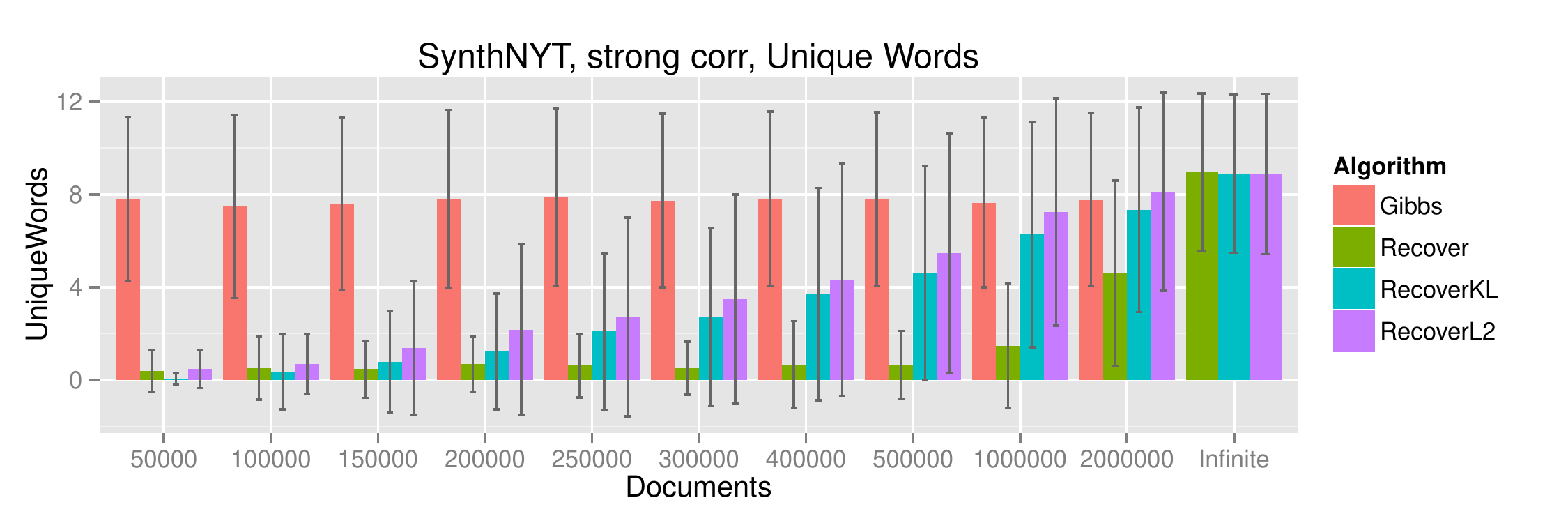}
\caption{Results for a semi-synthetic model generated from a model trained on NY Times articles with $K=100$, with stronger correlation between topics.}
\end{center}
\end{figure}

\begin{figure}[h]
\begin{center}
\includegraphics[scale=0.35]{nipsl1_error.pdf}
\includegraphics[scale=0.35]{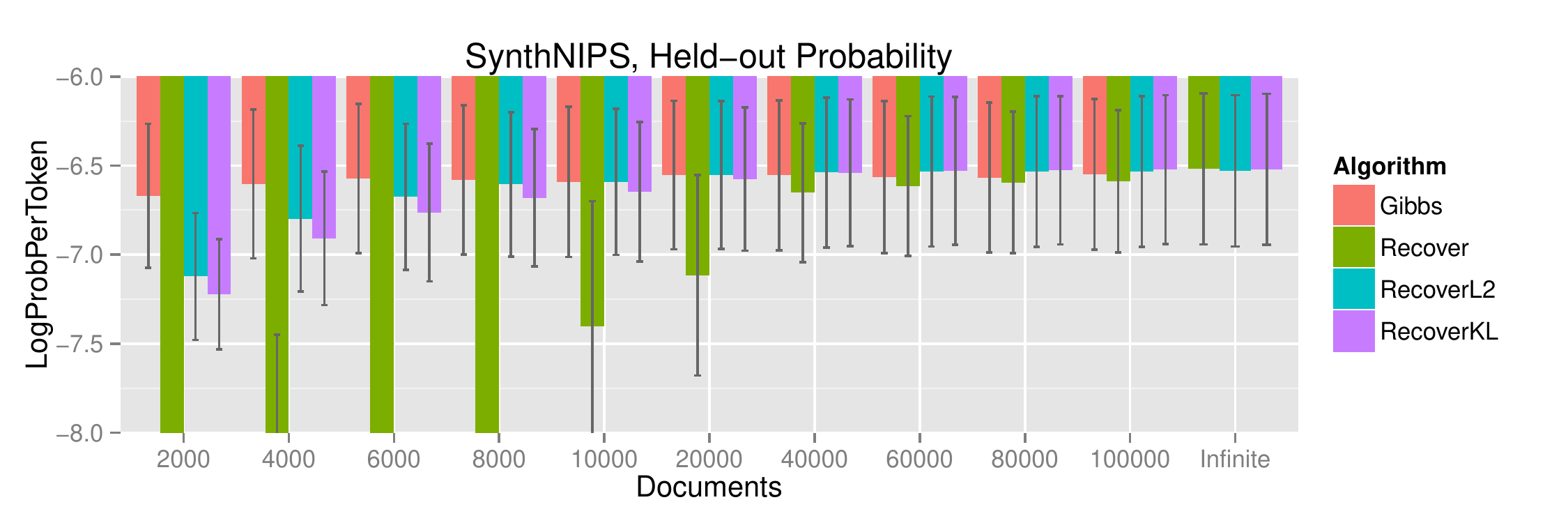}
\includegraphics[scale=0.35]{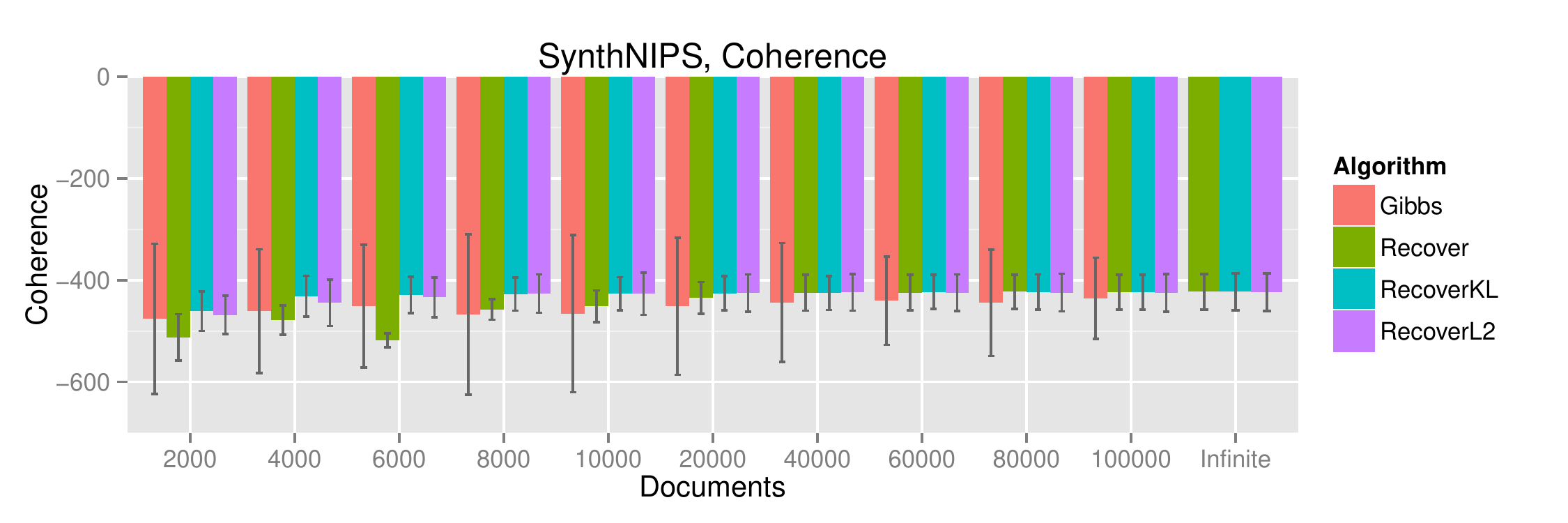}
\includegraphics[scale=0.35]{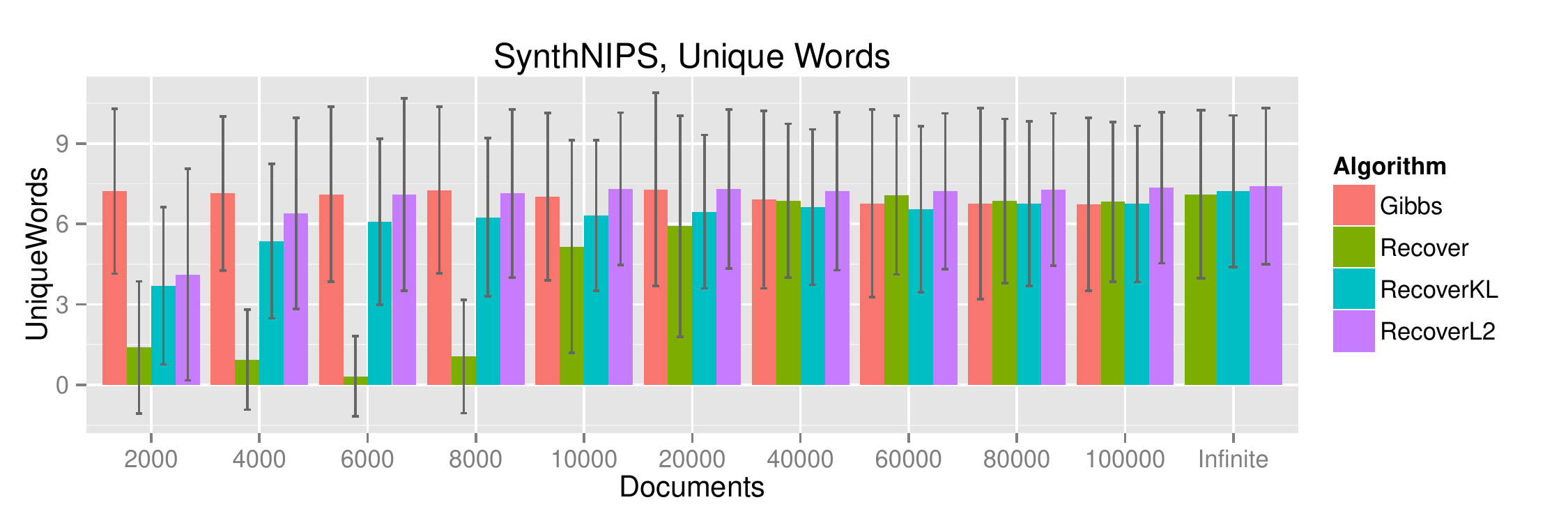}
\caption{Results for a semi-synthetic model generated from a model trained on NIPS papers with $K=100$. For $D \in \{2000, 6000, 8000\}$, Recover produces log probabilities of $-\infty$ for some held-out documents.}
\end{center}
\end{figure}

\subsection{Sample Topics}

Tables \ref{tbl:samples1}, \ref{tbl:samples2}, and \ref{tbl:samples3} show 100 topics trained on real NY Times articles using the RecoverL2 algorithm.
Each topic is followed by the most similar topic (measured by $\ell_1$ distance) from a model trained on the same documents with Gibbs sampling.
When the anchor word is among the top six words by probability it is highlighted in bold.
Note that the anchor word is frequently not the most prominent word.

\begin{table}[htdp]
\caption{Example topic pairs from NY Times sorted by $\ell_1$ distance, anchor words in bold.}
\begin{center}
\rowcolors{1}{}{gray!20}
{\tiny
\begin{tabular}{c|p{6cm}}
RecoverL2 & run inning game hit season zzz\_anaheim\_angel \\ 
Gibbs & run inning hit game ball pitch \\ 
\hline
RecoverL2 & king goal game team games season \\ 
Gibbs & point game team play season games \\ 
\hline
RecoverL2 & yard game play season team touchdown \\ 
Gibbs & yard game season team play quarterback \\ 
\hline
RecoverL2 & point game team season games play \\ 
Gibbs & point game team play season games \\ 
\hline
RecoverL2 & zzz\_laker point {\bf zzz\_kobe\_bryant} zzz\_o\_neal game team \\ 
Gibbs & point game team play season games \\ 
\hline
RecoverL2 & point game team season player {\bf zzz\_clipper} \\ 
Gibbs & point game team season play zzz\_usc \\ 
\hline
RecoverL2 & ballot election court votes vote zzz\_al\_gore \\ 
Gibbs & election ballot zzz\_florida zzz\_al\_gore votes vote \\ 
\hline
RecoverL2 & game zzz\_usc team play point season \\ 
Gibbs & point game team season play zzz\_usc \\ 
\hline
RecoverL2 & company billion companies percent million stock \\ 
Gibbs & company million percent billion analyst deal \\ 
\hline
RecoverL2 & car race team season driver point \\ 
Gibbs & race car driver racing zzz\_nascar team \\ 
\hline
RecoverL2 & zzz\_dodger season run inning right game \\ 
Gibbs & season team baseball game player yankees \\ 
\hline
RecoverL2 & palestinian zzz\_israeli zzz\_israel official attack {\bf zzz\_palestinian} \\ 
Gibbs & palestinian zzz\_israeli zzz\_israel attack zzz\_palestinian zzz\_yasser\_arafat \\ 
\hline
RecoverL2 & zzz\_tiger\_wood shot round player par play \\ 
Gibbs & zzz\_tiger\_wood shot golf tour round player \\ 
\hline
RecoverL2 & percent stock market companies fund quarter \\ 
Gibbs & percent economy market stock economic growth \\ 
\hline
RecoverL2 & zzz\_al\_gore {\bf zzz\_bill\_bradley} campaign president zzz\_george\_bush vice \\ 
Gibbs & zzz\_al\_gore zzz\_george\_bush campaign presidential republican zzz\_john\_mccain \\ 
\hline
RecoverL2 & zzz\_george\_bush {\bf zzz\_john\_mccain} campaign republican zzz\_republican voter \\ 
Gibbs & zzz\_al\_gore zzz\_george\_bush campaign presidential republican zzz\_john\_mccain \\ 
\hline
RecoverL2 & net team season point player {\bf zzz\_jason\_kidd} \\ 
Gibbs & point game team play season games \\ 
\hline
RecoverL2 & yankees run team season inning hit \\ 
Gibbs & season team baseball game player yankees \\ 
\hline
RecoverL2 & zzz\_al\_gore zzz\_george\_bush percent president campaign zzz\_bush \\ 
Gibbs & zzz\_al\_gore zzz\_george\_bush campaign presidential republican zzz\_john\_mccain \\ 
\hline
RecoverL2 & zzz\_enron company firm {\bf zzz\_arthur\_andersen} companies lawyer \\ 
Gibbs & zzz\_enron company firm accounting zzz\_arthur\_andersen financial \\ 
\hline
RecoverL2 & team play game yard season player \\ 
Gibbs & yard game season team play quarterback \\ 
\hline
RecoverL2 & film movie show director play character \\ 
Gibbs & film movie character play minutes hour \\ 
\hline
RecoverL2 & zzz\_taliban zzz\_afghanistan official zzz\_u\_s government military \\ 
Gibbs & zzz\_taliban zzz\_afghanistan zzz\_pakistan afghan zzz\_india government \\ 
\hline
RecoverL2 & palestinian zzz\_israel israeli peace zzz\_yasser\_arafat leader \\ 
Gibbs & palestinian zzz\_israel peace israeli zzz\_yasser\_arafat leader \\ 
\hline
RecoverL2 & point team game shot play zzz\_celtic \\ 
Gibbs & point game team play season games \\ 
\hline
RecoverL2 & zzz\_bush {\bf zzz\_mccain} campaign republican tax zzz\_republican \\ 
Gibbs & zzz\_al\_gore zzz\_george\_bush campaign presidential republican zzz\_john\_mccain \\ 
\hline
RecoverL2 & zzz\_met run team game hit season \\ 
Gibbs & season team baseball game player yankees \\ 
\hline
RecoverL2 & team game season play games win \\ 
Gibbs & team coach game player season football \\ 
\hline
RecoverL2 & government war {\bf zzz\_slobodan\_milosevic} official court president \\ 
Gibbs & government war country rebel leader military \\ 
\hline
RecoverL2 & game set player {\bf zzz\_pete\_sampras} play won \\ 
Gibbs & player game match team soccer play \\ 
\hline
RecoverL2 & zzz\_al\_gore campaign {\bf zzz\_bradley} president democratic zzz\_clinton \\ 
Gibbs & zzz\_al\_gore zzz\_george\_bush campaign presidential republican zzz\_john\_mccain \\ 
\hline
RecoverL2 & team zzz\_knick player season point play \\ 
Gibbs & point game team play season games \\ 
\hline
RecoverL2 & com web www information sport question \\ 
Gibbs & palm beach com statesman daily american \\ 
\end{tabular}
}
\end{center}
\label{tbl:samples1}
\end{table}%

\begin{table}[htdp]
\caption{Example topic pairs from NY Times sorted by $\ell_1$ distance, anchor words in bold.}
\begin{center}
\rowcolors{1}{}{gray!20}
{\tiny
\begin{tabular}{c|p{6cm}}
RecoverL2 & season team game coach play school \\ 
Gibbs & team coach game player season football \\ 
\hline
RecoverL2 & air shower rain wind storm front \\ 
Gibbs & water fish weather storm wind air \\ 
\hline
RecoverL2 & book film {\bf beginitalic} enditalic look movie \\ 
Gibbs & film movie character play minutes hour \\ 
\hline
RecoverL2 & zzz\_al\_gore campaign election zzz\_george\_bush zzz\_florida president \\ 
Gibbs & zzz\_al\_gore zzz\_george\_bush campaign presidential republican zzz\_john\_mccain \\ 
\hline
RecoverL2 & race won horse {\bf zzz\_kentucky\_derby} win winner \\ 
Gibbs & horse race horses winner won zzz\_kentucky\_derby \\ 
\hline
RecoverL2 & company companies {\bf zzz\_at} percent business stock \\ 
Gibbs & company companies business industry firm market \\ 
\hline
RecoverL2 & company million companies percent business customer \\ 
Gibbs & company companies business industry firm market \\ 
\hline
RecoverL2 & team coach season player jet job \\ 
Gibbs & team player million season contract agent \\ 
\hline
RecoverL2 & season team game play player zzz\_cowboy \\ 
Gibbs & yard game season team play quarterback \\ 
\hline
RecoverL2 & zzz\_pakistan zzz\_india official group attack zzz\_united\_states \\ 
Gibbs & zzz\_taliban zzz\_afghanistan zzz\_pakistan afghan zzz\_india government \\ 
\hline
RecoverL2 & show network night television zzz\_nbc program \\ 
Gibbs & film movie character play minutes hour \\ 
\hline
RecoverL2 & com information question zzz\_eastern commentary daily \\ 
Gibbs & com question information zzz\_eastern daily commentary \\ 
\hline
RecoverL2 & power plant company percent million energy \\ 
Gibbs & oil power energy gas prices plant \\ 
\hline
RecoverL2 & cell stem research zzz\_bush human patient \\ 
Gibbs & cell research human scientist stem genes \\ 
\hline
RecoverL2 & {\bf zzz\_governor\_bush} zzz\_al\_gore campaign tax president plan \\ 
Gibbs & zzz\_al\_gore zzz\_george\_bush campaign presidential republican zzz\_john\_mccain \\ 
\hline
RecoverL2 & cup minutes add tablespoon water oil \\ 
Gibbs & cup minutes add tablespoon teaspoon oil \\ 
\hline
RecoverL2 & family home book right com children \\ 
Gibbs & film movie character play minutes hour \\ 
\hline
RecoverL2 & zzz\_china chinese zzz\_united\_states {\bf zzz\_taiwan} official government \\ 
Gibbs & zzz\_china chinese zzz\_beijing zzz\_taiwan government official \\ 
\hline
RecoverL2 & death court law case lawyer zzz\_texas \\ 
Gibbs & trial death prison case lawyer prosecutor \\ 
\hline
RecoverL2 & company percent million sales business companies \\ 
Gibbs & company companies business industry firm market \\ 
\hline
RecoverL2 & dog jump show quick brown {\bf fox} \\ 
Gibbs & film movie character play minutes hour \\ 
\hline
RecoverL2 & {\bf shark} play team attack water game \\ 
Gibbs & film movie character play minutes hour \\ 
\hline
RecoverL2 & anthrax official mail letter worker attack \\ 
Gibbs & anthrax official letter mail nuclear chemical \\ 
\hline
RecoverL2 & president zzz\_clinton zzz\_white\_house zzz\_bush official zzz\_bill\_clinton \\ 
Gibbs & zzz\_bush zzz\_george\_bush president administration zzz\_white\_house zzz\_dick\_cheney \\ 
\hline
RecoverL2 & father family {\bf zzz\_elian} boy court zzz\_miami \\ 
Gibbs & zzz\_cuba zzz\_miami cuban zzz\_elian boy protest \\ 
\hline
RecoverL2 & oil prices percent million market zzz\_united\_states \\ 
Gibbs & oil power energy gas prices plant \\ 
\hline
RecoverL2 & {\bf zzz\_microsoft} company computer system window software \\ 
Gibbs & zzz\_microsoft company companies cable zzz\_at zzz\_internet \\ 
\hline
RecoverL2 & government election zzz\_mexico political zzz\_vicente\_fox president \\ 
Gibbs & election political campaign zzz\_party democratic voter \\ 
\hline
RecoverL2 & fight {\bf zzz\_mike\_tyson} round right million champion \\ 
Gibbs & fight zzz\_mike\_tyson ring fighter champion round \\ 
\hline
RecoverL2 & right law president zzz\_george\_bush zzz\_senate {\bf zzz\_john\_ashcroft} \\ 
Gibbs & election political campaign zzz\_party democratic voter \\ 
\hline
RecoverL2 & com home look found show www \\ 
Gibbs & film movie character play minutes hour \\ 
\hline
RecoverL2 & car driver race {\bf zzz\_dale\_earnhardt} racing zzz\_nascar \\ 
Gibbs & night hour room hand told morning \\ 
\hline
RecoverL2 & book women family called author woman \\ 
Gibbs & film movie character play minutes hour \\ 
\end{tabular}
}
\end{center}
\label{tbl:samples2}
\end{table}%

\begin{table}[htdp]
\caption{Example topic pairs from NY Times sorted by $\ell_1$ distance, anchor words in bold.}
\begin{center}
\rowcolors{1}{}{gray!20}
{\tiny
\begin{tabular}{c|p{6cm}}
RecoverL2 & tax bill zzz\_senate billion plan zzz\_bush \\ 
Gibbs & bill zzz\_senate zzz\_congress zzz\_house legislation zzz\_white\_house \\ 
\hline
RecoverL2 & company {\bf francisco} san com food home \\ 
Gibbs & palm beach com statesman daily american \\ 
\hline
RecoverL2 & team player season game {\bf zzz\_john\_rocker} right \\ 
Gibbs & season team baseball game player yankees \\ 
\hline
RecoverL2 & zzz\_bush official zzz\_united\_states zzz\_u\_s president zzz\_north\_korea \\ 
Gibbs & zzz\_united\_states weapon zzz\_iraq nuclear zzz\_russia zzz\_bush \\ 
\hline
RecoverL2 & zzz\_russian zzz\_russia official military war attack \\ 
Gibbs & government war country rebel leader military \\ 
\hline
RecoverL2 & wine {\bf wines} percent zzz\_new\_york com show \\ 
Gibbs & film movie character play minutes hour \\ 
\hline
RecoverL2 & police {\bf zzz\_ray\_lewis} player team case told \\ 
Gibbs & police officer gun crime shooting shot \\ 
\hline
RecoverL2 & government group political tax leader money \\ 
Gibbs & government war country rebel leader military \\ 
\hline
RecoverL2 & percent company million airline flight deal \\ 
Gibbs & flight airport passenger airline security airlines \\ 
\hline
RecoverL2 & book ages children school boy web \\ 
Gibbs & book author writer word writing read \\ 
\hline
RecoverL2 & {\bf corp} group president energy company member \\ 
Gibbs & palm beach com statesman daily american \\ 
\hline
RecoverL2 & team tour {\bf zzz\_lance\_armstrong} won race win \\ 
Gibbs & zzz\_olympic games medal gold team sport \\ 
\hline
RecoverL2 & priest church official abuse bishop sexual \\ 
Gibbs & church religious priest zzz\_god religion bishop \\ 
\hline
RecoverL2 & human drug company companies million scientist \\ 
Gibbs & scientist light science planet called space \\ 
\hline
RecoverL2 & music {\bf zzz\_napster} company song com web \\ 
Gibbs & palm beach com statesman daily american \\ 
\hline
RecoverL2 & death government case federal official {\bf zzz\_timothy\_mcveigh} \\ 
Gibbs & trial death prison case lawyer prosecutor \\ 
\hline
RecoverL2 & million shares offering public company initial \\ 
Gibbs & company million percent billion analyst deal \\ 
\hline
RecoverL2 & buy {\bf panelist} thought flavor product ounces \\ 
Gibbs & food restaurant chef dinner eat meal \\ 
\hline
RecoverL2 & school student program teacher public children \\ 
Gibbs & school student teacher children test education \\ 
\hline
RecoverL2 & security official government airport federal bill \\ 
Gibbs & flight airport passenger airline security airlines \\ 
\hline
RecoverL2 & company member credit card money mean \\ 
Gibbs & zzz\_enron company firm accounting zzz\_arthur\_andersen financial \\ 
\hline
RecoverL2 & million percent bond tax debt bill \\ 
Gibbs & million program billion money government federal \\ 
\hline
RecoverL2 & million company zzz\_new\_york business art percent \\ 
Gibbs & art artist painting museum show collection \\ 
\hline
RecoverL2 & percent million number official group black \\ 
Gibbs & palm beach com statesman daily american \\ 
\hline
RecoverL2 & company tires million car zzz\_ford percent \\ 
Gibbs & company companies business industry firm market \\ 
\hline
RecoverL2 & article zzz\_new\_york {\bf misstated} company percent com \\ 
Gibbs & palm beach com statesman daily american \\ 
\hline
RecoverL2 & company million percent companies government official \\ 
Gibbs & company companies business industry firm market \\ 
\hline
RecoverL2 & official million train car system plan \\ 
Gibbs & million program billion money government federal \\ 
\hline
RecoverL2 & {\bf test} student school look percent system \\ 
Gibbs & patient doctor cancer medical hospital surgery \\ 
\hline
RecoverL2 & con una mas dice las anos \\ 
Gibbs & fax syndicate article com information con \\ 
\hline
RecoverL2 & {\bf por} con una mas millones como \\ 
Gibbs & fax syndicate article com information con \\ 
\hline
RecoverL2 & las como {\bf zzz\_latin\_trade} articulo telefono fax \\ 
Gibbs & fax syndicate article com information con \\ 
\hline
RecoverL2 & {\bf los} con articulos telefono representantes zzz\_america\_latina \\ 
Gibbs & fax syndicate article com information con \\ 
\hline
RecoverL2 & {\bf file} sport read internet email zzz\_los\_angeles \\ 
Gibbs & web site com www mail zzz\_internet \\ 
\end{tabular}
}
\end{center}
\label{tbl:samples3}
\end{table}%

\section{Algorithmic Details}

\subsection{Generating $Q$ matrix}

For each document, let $H_d$ be the vector in $\R^V$ such that the $i$-th entry is the number of times word $i$ appears in
document $d$, $n_d$ be the length of the document and $W_d$ be the topic vector chosen according to Dirichlet distribution when the documents are generated.
Conditioned on $W_d$'s, our algorithms require the expectation of $Q$ to be $\frac{1}{M} \sum_{d=1}^M AW_dW_d^TA^T$.

In order to achieve this, similar to \citep{AFHKL}, let the normalized vector $\tilde{H}_d = \frac{H_d}{\sqrt{n_d(n_d-1)}}$ and diagonal matrix
$\hat{H}_d = \frac{\mbox{Diag}(H_d) }{n_d(n_d-1)}$. Compute the matrix
$$\tilde{H}_d \tilde{H}_d^T - \hat{H}_d = \frac{1}{n_d(n_d-1)}\sum_{i\ne j, i,j\in[n_d]} e_{z_{d,i}}e_{z_{d,j}}^T.$$
Here $z_{d,i}$ is the $i$-th word of document $d$, and $e_i \in \R^V$ is the basis vector. From the generative model,
the expectation of all terms $e_{z_{d,i}}e_{z_{d,j}}^T$ are equal to $AW_dW_d^T A^T$, hence by linearity of expectation
we know $\E[\tilde{H}_d \tilde{H}_d^T - \hat{H}_d] = AW_dW_d^T A^T.$

If we collect all the column vectors $\tilde{H}_d$ to form a large sparse matrix $\tilde{H}$, and compute the sum of all $\hat{H}_d$ to get the diagonal matrix $\hat{H}$, we know $Q = \tilde{H}\tilde{H}^T - \hat{H}$ has the desired expectation.
The running time of this step is $O(MD^2)$ where $D^2$ is the expectation of the length of the document squared. 

\subsection{Exponentiated gradient algorithm}
The optimization problem that arises in RecoverKL and RecoverL2 has the following form,
\begin{gather*}
\text{minimize } d(b, Tx)\\
\text{subject to: } x \geq 0 \text{ and } x^T{\bf1} = 1
\end{gather*}
where $d(\cdot,\cdot)$ is a Bregman divergence, $x$ is a vector of length $K$, and $T$ is a matrix of size $V \times K$.
We solve this optimization problem using the Exponentiated Gradient algorithm \citep{Kivinen95exponentiatedgradient}, described in Algorithm~\ref{alg:gradientDescent}. In our experiments we show results using both squared Euclidean distance and KL divergence for the divergence measure. Stepsizes are chosen with a line search to find an $\eta$ that satisfies the Wolfe and Armijo conditions (For details, see \citet{Nocedal2006NO}). 
We test for convergence using the KKT conditions. Writing the KKT conditions for our constrained minimization problem:
\begin{enumerate}
\item Stationarity: $\nabla_x d(b, Tx^{*}) - \vec{\lambda} + \mu {\bf 1}$ = 0
\item Primal Feasibility: $x^{*} \geq 0$, $|x^{*}|_1 = 1$
\item Dual Feasibility: $\lambda \geq 0$
\item Complementary Slackness: $\lambda_i x_i^{*} = 0$
\end{enumerate}
For every iterate of $x$ generated by Exponentiated Gradient, we set $\lambda, \mu$ to satisfy conditions 1-3. This gives the following equations: 
\begin{gather*}
\lambda = \nabla_x d(b, Tx^{*}) + \mu {\bf 1}\\
\mu = -(\nabla_x d(b, Tx^{*}))_{\min}
\end{gather*}
By construction conditions 1-3 are satisfied (note that the multiplicative update and the projection step ensure that $x$ is always primal feasible). Convergence is tested by checking whether the final KKT condition holds within some tolerance. Since $\lambda$ and $x$ are nonnegative, we check complimentary slackness by testing whether $\lambda^Tx < \epsilon$. This convergence test can also be thought of as testing the value of the primal-dual gap, since the Lagrangian function has the form: $L(x, \lambda, \mu) = d(b, Tx) - \lambda^Tx + \mu(x^T{\bf 1} - 1)$,  and $(x^T{\bf 1} - 1)$ is zero at every iteration. 

\begin{algorithm}[t]
\caption{Exponentiated Gradient}
\label{alg:gradientDescent}    
\begin{algorithmic}
\Require Matrix $T$, vector $b$, divergence measure $d(\cdot,\cdot)$, tolerance parameter $\epsilon$
\Ensure non-negative normalized vector $x$ close to $x^*$, the minimizer of $d(b, Tx)$
\State Initialize $x \gets \frac{1}{K}{\bf 1}$
\State Initialize Converged $\gets$ False
\While{not Converged}
\State $p = \nabla d(b, Tx)$
\State Choose a step size $\eta_t$
\State $x \gets x e^{-\eta_t p}$ (Gradient step)
\State $x \gets \frac{x}{|x|_1}$ (Projection onto the simplex)
\State $\mu \gets \nabla d(b, Tx)_{\min}$
\State $\lambda \gets \nabla d(b, Tx) - \mu$
\State Converged $\gets \lambda^Tx < \epsilon$ 
\EndWhile
\end{algorithmic}
\end{algorithm}

The running time of RecoverL2 is the time of solving $V$ small ($K\times K$) quadratic programs. Especially when using Exponentiated Gradient to solve the quadratic program, each word requires $O(KV)$ time for preprocessing and $O(K^2)$ per iteration. The total running time is $O(KV^2+K^2V T)$ where $T$ is the average number of iterations. The value of $T$ is about $100-1000$ depending on data sets.
\vfill

\end{document}